\documentclass{article} %
\usepackage{iclr2026_conference,times}
\usepackage{xcolor}         %
\usepackage{xparse}
\usepackage{graphicx}
\usepackage{subcaption}
\usepackage{amsmath}
\usepackage{amsthm}
\usepackage{amssymb}
\usepackage{bm}
\usepackage{booktabs}
\usepackage{multirow}
\usepackage{wrapfig}
\usepackage{makecell}

\newtheorem{definition}{Definition}
\newtheorem{theorem}{Theorem}

\newtheorem{lemma}{Lemma}

\newtheorem{remark}{Remark}
\newtheorem*{thm*}{Theorem}
\newtheorem*{lem*}{Lemma}

\usepackage{algpseudocode}
\usepackage[linesnumbered,ruled,vlined]{algorithm2e}
\SetKwComment{tcp}{$\triangleright$\ }{}
\DontPrintSemicolon
\SetNoFillComment

\usepackage{mathtools}
\usepackage{subfiles} 
\usepackage{tcolorbox}

\usepackage[utf8]{inputenc} %
\usepackage[T1]{fontenc}    %
\usepackage{hyperref}       %
\usepackage{url}            %
\usepackage{booktabs}       %
\usepackage{amsfonts}       %
\usepackage{nicefrac}       %
\usepackage{microtype}      %
\usepackage{cleveref}       %
\usepackage{lipsum}         %
\usepackage{graphicx}
\usepackage{natbib}
\usepackage{doi}
\usepackage{enumitem}
\usepackage{bm}

\newcommand{\dprag}{\textsc{DP-RAG}}
\newcommand{\topk}{\textsc{Top-K}}
\newcommand{\naivedprag}{\textsc{Naive-Multi-RAG}}
\newcommand{\sampledprag}{\textsc{Subsampling-Multi-RAG}}
\newcommand{\dpfixtau}{\textsc{MuRAG}}
\newcommand{\dpadaptovetau}{\textsc{MuRAG-Ada}}
\newcommand{\algocount}{\textsc{Count}}
\newcommand{\expomech}{\textsc{ExpoMech}}
\newcommand{\poissionsampling}{\textsc{PoissonSampling}}

\newcommand{\llm}{\mathrm{LLM}}

\SetKwInput{KwSet}{Set}
\SetKwInput{KwConstruct}{Construct}
\SetKwInput{KwRequire}{Require}
\SetKwInOut{KwOut}{output}

\usepackage{amsmath,amsfonts,bm}

\def\eqref#1{equation~\ref{#1}}

\def\1{\bm{1}}

\DeclareMathAlphabet{\mathsfit}{\encodingdefault}{\sfdefault}{m}{sl}
\SetMathAlphabet{\mathsfit}{bold}{\encodingdefault}{\sfdefault}{bx}{n}

\title{Private-RAG: Answering Multiple Queries with LLMs while Keeping Your Data Private}

\author{
\hspace{-0.5em} Ruihan Wu\thanks{Equal contribution.} \\
Computer Science and Engineering \\
University of California, San Diego \\
\texttt{ruw076@ucsd.edu}
\And
Erchi Wang\footnotemark[1] \\
Halıcıoğlu Data Science Institute \\
University of California, San Diego \\
\texttt{erw011@ucsd.edu}
\AND
Zhiyuan Zhang \\
Department of Computer Science \\
University of California, Los Angeles \\
\texttt{hollyzhang03@ucla.edu}
\And
\qquad \qquad \qquad \hspace{1.5em} Yu-Xiang Wang \\
\qquad \qquad \qquad \hspace{1.5em} Halıcıoğlu Data Science Institute \\
\qquad \qquad \qquad \hspace{1.5em} University of California, San Diego \\
\qquad \qquad \qquad \hspace{1.5em} \texttt{yuxiangw@ucsd.edu}
}

\iclrfinalcopy %
\begin{document}

\maketitle

\begin{abstract}

Retrieval-augmented generation (RAG) enhances large language models (LLMs) by retrieving documents from an external corpus at inference time. When this corpus contains sensitive information, however, unprotected RAG systems are at risk of leaking private information. Prior work has introduced differential privacy (DP) guarantees for RAG, but only in single-query settings, which fall short of realistic usage. In this paper, we study the more practical multi-query setting and propose two DP-RAG algorithms. The first, \dpfixtau, leverages an individual privacy filter so that the accumulated privacy loss only depends on how frequently each document is retrieved rather than the total number of queries. The second, \dpadaptovetau, further improves utility by privately releasing query-specific thresholds, enabling more precise selection of relevant documents. Our experiments across multiple LLMs and datasets demonstrate that the proposed methods scale to hundreds of queries within a practical DP budget ($\varepsilon\approx10$), while preserving meaningful utility.

\end{abstract}

\section{Introduction}
Retrieval-augmented generation (RAG) has become a popular approach for deploying large language models (LLMs) in real-world applications. 
A core feature of RAG is its reliance on an external dataset as the primary knowledge source at inference time. 
For example, a medical RAG system may retrieve historical patient records to answer clinical questions more accurately.
However, such external datasets often contain sensitive or confidential information. 
In domains like healthcare or law, the retrieved content may expose private records, raising serious privacy concerns. 
Prior work has shown that RAG systems without proper safeguards are vulnerable to information leakage \citep{naseh2025riddle, liu2025mask, anderson2024my, li2025generating, zhang2025deal, zeng2024good, jiang2024rag, peng2024data}, compromising data owner privacy and user trust.

Differential privacy (DP) is a widely adopted framework for providing rigorous guarantees on individual data protection. Recent work~\citep{koga2024privacy} has proposed DPSparseVoteRAG, a RAG system that ensures the generated answer satisfies DP with respect to the external dataset, for \emph{a single user query}.
Empirical results demonstrate that this approach outperforms the baseline using a public LLM without the external dataset, while achieving an \(\varepsilon\)-DP guarantee with \(\varepsilon \approx 10\).

In realistic deployments, many queries may be issued by one or more users. 
A naïve approach that applies DPSparseVoteRAG to each query and relies on standard composition theorems quickly exhausts a reasonable privacy budget. As our experimental results (Figure~\ref{fig:main}) show, to achieve reasonable utility, this approach may require a privacy budget as large as $\varepsilon=1000$, which is generally considered too weak. This raises a key question:

\emph{Can we design a differentially private RAG algorithm that handles hundreds of queries while ensuring both meaningful privacy and utility?}

We answer this question affirmatively and summarize our contributions below.

\textbf{Circumventing Query-Composition Overhead with Per-Document Rényi Filters.}\quad We propose a novel framework for multi-query differentially private RAG. Rather than composing a sequence of single-query DP-RAG executions, where the privacy budget grows with the number of queries, we leverage \emph{individual R'enyi filters}~\citep{feldman2021individual}. These filters bound privacy loss based on how many times each document is retrieved, yielding substantial savings when queries access largely disjoint documents. To the best of our knowledge, this is the first application of privacy filters in the RAG setting. Our framework can incorporate any single-query private RAG algorithm.

\textbf{Two DP Multi-RAG Algorithms for Varying Test Query Dependencies.}\quad
We propose two differentially private RAG algorithms for the multi-query setting through threshold-based screening of relevant documents and their are tailored to the degree of relevance among test-time queries.
\dpfixtau\ (Algorithm~\ref{alg:dp_fix_tau}) uses a fixed relevance threshold across all queries and is sufficient to work well for settings where queries are independent and do not share relevant private documents.
\dpadaptovetau\ (Algorithm~\ref{alg:batch_dp_rag-v2}) allocates a small portion of the privacy budget to release a query-specific relevance threshold, enabling more efficient use of the budget when queries are related and share overlapping relevant documents.
    
\textbf{Practical Multi-Query RAG with Non-Trivial Privacy Guarantees.}\quad 
We evaluate our algorithms through extensive experiments on three LLMs (OPT-1.3B, Pythia-1.4B, and Mistral-7B). Our evaluation spans three types of datasets: standard RAG benchmarks (\textit{Natural Questions}, \textit{Trivia Questions}), a more challenging multi-hop QA dataset (MQuAKE) with correlated questions, and a privacy-sensitive application (ChatDoctor) consisting of patient–doctor QA pairs. Empirical results show that both of our methods can answer hundreds of queries within a total privacy budget of $\varepsilon \approx 10$ while maintaining reasonable utility, a trade-off no baseline method achieves. Furthermore, we demonstrate that our approaches with $\varepsilon=10$ effectively defend against a state-of-the-art multi-query membership inference attack for RAG.

\section{Differential Private Retrieval-augmented Generation}

\label{sec:prelim}
\textbf{Notation.}\quad 
Let $\mathcal{V}$ denote a finite vocabulary, and let $x \in \mathcal{V}^*$ represent a prompt of arbitrary length. A document set of arbitrary size is denoted by $D = \{z_1, z_2, \ldots\}$, where each document $z_i \in \mathcal{V}^*$. For convenience, we denote by $\mathcal{Z}$ the document space, i.e., the set of all finite-length sequences over $\mathcal{V}$.

\textbf{Differential Privacy.}\quad 
We denote the data space by $\mathcal{X}$. Two datasets $D, D' \in \mathcal{X}^*$ are said to be neighboring if they differ in at most one element. In this work, we study \emph{document-level privacy} under the add/remove neighboring relation, where the data universe is $\mathcal{V}^*$ and two datasets are neighbors if they differ by exactly one document.

\begin{definition}[Differential Privacy \citep{dmns06}]
A randomized algorithm $\mathcal{M}:\mathcal{X}^* \rightarrow \Omega$ satisfies $(\varepsilon, \delta)$-differential privacy if, for all neighboring datasets $X, X' \in \mathcal{X}^*$ and all measurable subsets $O \subseteq \Omega$, $\Pr[\mathcal{M}(X) \in O] \leq e^{\varepsilon}\Pr[\mathcal{M}(X') \in O] + \delta$.
\end{definition}

\begin{definition}[R\'enyi Differential Privacy \citep{mironov2017renyi}]
A randomized algorithm $\mathcal{M}:\mathcal{X}^* \rightarrow \Omega$ satisfies $(\alpha, \varepsilon)$-R\'enyi Differential Privacy (RDP) if, for all neighboring datasets $X, X' \in \mathcal{X}^*$, the R\'enyi divergence of order $\alpha > 1$ between $\mathcal{M}(X)$ and $\mathcal{M}(X')$ is at most $\varepsilon$, i.e. $D_\alpha(\mathcal{M}(X)\,\|\,\mathcal{M}(X')) \leq \varepsilon$.
\end{definition}

We may also consider \emph{individual-level} RDP, where the R\'enyi divergence is evaluated on neighboring datasets that differ in a particular data point $z_i$. Let $\mathcal{S}(z_i,n)$ denote the set of dataset pairs $(S,\Tilde{S})$ such that $|S|,|\Tilde{S}| < n$ and $z_i \in S \triangle \Tilde{S}$, i.e. exactly one of $S,\Tilde{S}$ contains $z_i$.

\begin{definition}[Individual R\'enyi Differential Privacy]
A randomized algorithm $\mathcal{M}:\mathcal{X}^* \rightarrow \Omega$ satisfies $(\alpha,\varepsilon)$-individual RDP at point $z_i$ if, for all $(X,X') \in \mathcal{S}(z_i,n)$, $D_\alpha(\mathcal{M}(X)\,\|\,\mathcal{M}(X')) \leq \varepsilon$
\end{definition}

A \emph{privacy filter} is a stopping rule that tracks cumulative privacy loss and halts execution once the privacy budget is exceeded, thereby ensuring that the designed privacy guarantees are never violated. For completeness, we briefly introduce individual RDP filters; for a rigorous treatment, we refer readers to \citet{feldman2021individual}.

\begin{definition}[(Individual) R\'enyi Differential Privacy Filters \citep{feldman2021individual}]
A random variable \(\mathcal{F}_{\alpha,B}:\Omega^* \rightarrow \{\mathrm{CONT},\mathrm{HALT}\}\) is a privacy filter for \((\alpha,B)\)-RDP if it halts the execution of an algorithm before its accumulated (individual) privacy loss, measured in \(\alpha\)-R\'enyi divergence, exceeds \(B\).
\end{definition}

\textbf{Problem Setting.}\quad 
We study retrieval-augmented generation (RAG) with a sensitive external document collection. A decoder-only LLM with greedy decoding is modeled as a function \(\mathrm{LLM}:\mathcal{V}^* \times \mathcal{Z} \rightarrow \mathcal{V}\). Given a user prompt $x \in \mathcal{V}^*$, the system retrieves a subset of documents $D_x = R_k(x, D)$ from a private external corpus $D \in \mathcal{Z}$, where the retrieval function $R_k:\mathcal{V}^* \times \mathcal{Z} \rightarrow \mathcal{Z}$ returns the $k$ most relevant documents. The corpus $D$ contains sensitive documents, each potentially corresponding to private user information.  

We adopt a threat model in which the adversary has no direct access to the corpus $D$ but may issue arbitrary prompts $x$ to the RAG system. The underlying LLM is assumed to be public and independent of $D$. Our objective is to design a differentially private RAG mechanism that, given a set of queries $\{q_1, \dots, q_T\}$, the sensitive corpus $D$, a public LLM, and a total privacy budget $\varepsilon$, generates high-utility responses while guaranteeing $\varepsilon$-differential privacy with respect to corpus $D$.

\section{Methodology}
\subsection{Technical Overview}\label{sec:tech_review}

\textbf{Improved Privacy Accounting via Per-Document Privacy Filters.}\quad 
In retrieval-augmented generation (RAG), each query interacts with only a small, query-specific subset of the corpus $D$. 
This sparsity implies that most documents are accessed only rarely\footnote{We provide a more detailed discussion of this sparsity in Appendix~\ref{apx:sparse}.}. 
We leverage this by introducing a per-document privacy filter that monitors cumulative privacy loss and blocks further retrieval once a document’s budget is exhausted. 
Because privacy cost is incurred only upon retrieval, this accounting scheme naturally scales with the frequency of document access rather than the total number of queries.

\textbf{Screening Relevant Documents via Relevance Thresholding.}\quad 
If RAG were applied directly to the entire corpus, every document would be touched by each query, and per-document privacy filters would provide no benefit. 
To prevent this, \dpfixtau\ employs a global relevance threshold $\tau$\footnote{Intuitively, the threshold $\tau$ can be viewed as a chosen percentile of the relevance score distribution for a given query, ensuring that only the top-ranked documents contribute to privacy cost.
}: only documents whose scores exceed $\tau$ are retrieved and incur privacy cost. 
A document is excluded from all future retrievals once its privacy budget is exhausted. 
Since $\tau$ is fixed in advance and independent of the data, introducing this threshold does not consume additional privacy budget.

\textbf{Handling Correlated Queries via Adaptive Thresholding.}\quad 
When queries are \emph{correlated}, meaning their sets of relevant documents substantially overlap, a fixed relevance threshold $\tau$ can lead to inefficiencies. 
Specifically, since the relevance score distribution may shift across queries, a uniform threshold can cause some queries to retrieve more documents than necessary, prematurely exhausting the budgets of relevant documents and limiting their availability for later queries. 
To mitigate this, we propose \dpadaptovetau, which privately selects a query-specific threshold $\tau_t$ tailored to the relevance distribution of each query. 
By combining per-document privacy accounting with the private release of cumulative statistics, $\dpadaptovetau$ restricts retrieval to the most relevant documents, thereby reducing unnecessary budget consumption and preserving utility across correlated queries.

\textbf{Single-Query DP RAG after Screening.}\quad 
After thresholding, per-document privacy filters ensure that each retrieved document incurs loss only when used and is removed once its budget is exhausted. 
The resulting set is then passed to a single-query DP-RAG algorithm to generate the response. 
As shown in Algorithms~\ref{alg:dp_fix_tau} and~\ref{alg:batch_dp_rag-v2}, our multi-query framework is modular, supporting any private single-query RAG method. 
In this work, we instantiate it with a pure-DP variant of the algorithm from \citet{koga2024privacy} (Algorithm~\ref{alg:dp-rag-v2}).

\subsection{DP-RAG with a Fixed Threshold}\label{sec:dpfixtau}
In $\dpfixtau$, we impose a fixed relevance threshold $\tau$ to screen documents before retrieval. 
The threshold can either be publicly specified or privately estimated using a small portion of the privacy budget. 
The complete procedure is summarized in Algorithm~\ref{alg:dp_fix_tau} and the privacy guarantee is given in Theorem~\ref{lem:privacy_batch_rag_fix}. 
At a high level, the algorithm maintains a per-document privacy budget that is decremented whenever the document is retrieved. 
For each query, it first updates the active set of documents and then filters out most documents with scores below $\tau$. 
Among the remaining documents, the top-$k$ are selected by relevance, and a differentially private single-query RAG procedure is invoked to generate the response. 

Since whether a document exceeds the constant threshold $\tau$ depends only on its own score and not on the scores of other documents, the use of \textit{(Individual) R\'enyi Differential Privacy Filters} is valid. Consequently, for each query, privacy loss is charged only to the small subset of documents that pass the threshold, using a per-query budget $\varepsilon_q$, rather than to the entire corpus. The privacy guarantee of \dpfixtau{} is stated in Theorem~\ref{lem:privacy_batch_rag_fix}, and the proof is deferred to Appendix~\ref{apx:privacy_proof}.

\begin{theorem}[Privacy Guarantee of Algorithm~\ref{alg:dp_fix_tau}]\label{lem:privacy_batch_rag_fix}
    $\dpfixtau$ satisfies $\varepsilon$-differential privacy provided that the initial privacy budget assigned to each document $z \in D$ is at most $\varepsilon$. 
\end{theorem}

\begin{algorithm}[H]
\small
\caption{\dpfixtau : Differentially Private \textbf{Mu}lti-Query \textbf{R}etrieval-\textbf{A}ugmented \textbf{G}eneration}\label{alg:dp_fix_tau}
\setcounter{AlgoLine}{0}
\KwIn{Private dataset $D$, sequence of queries $\{q_1,\dots,q_T\}$, per-query DP budget $\varepsilon_q$, $\#$retrieved documents $k$, maximum retrievals per document $M$, relevance threshold $\tau$} 
\KwSet{Initialize individual budget for each document $z \in D$: $\mathcal{E}(z) = M \cdot \varepsilon_q$}
\For{$t = 1,\dots,T$}{
    $A_t = \{z \in D \mid \mathcal{E}(z) \geq \varepsilon_q\}$ \tcp*[r]{Update active document set}
    
    $D_{q_t} = \{z \in A_t \mid r(z, q_t) > \tau\}$ \tcp*[r]{Filter relevant documents}
    
    \For{$z \in D_{q_t}$}{
        $\mathcal{E}(z) \gets \mathcal{E}(z) - \varepsilon_q$ \tcp*[r]{Update budget for retrieved documents}
    }
    
    $D_{q_t}^k = \topk(D_{q_t}, k, r(\cdot, q_t))$ \tcp*[r]{Select top-$k$ relevant documents}
    
    $a_t = \dprag(x, D_{q_t}^k, \llm, \varepsilon_q)$ \tcp*[r]{Generate DP response via Algo.~\ref{alg:dp-rag-v2}}
}
\KwRet{$(a_1, \dots, a_T)$}
\end{algorithm}

\subsection{DP-RAG with Adaptive Threshold}\label{sec:adaptive_tau}
The score distribution can vary substantially across different questions, making a single global threshold ineffective. 
To guarantee the performance of single-query DP-RAG, the threshold must be set low enough to retrieve sufficient documents for all queries. 
However, this often results in many unnecessary documents being retrieved: although single-query DP-RAG uses at most $K$ documents, any additional documents above $K$ still incur privacy loss, wasting budget on unused data. This inefficiency can significantly degrade performance when those documents are needed by later queries. To overcome this limitation, we propose \dpadaptovetau{}, which privately releases a query-specific threshold $\tau_t$ adapted to the relevance distribution of each query.

The adaptive procedure works by discretizing the relevance scores into bins and then releasing noisy prefix sums until the cumulative count of retrieved documents exceeds $K$. 
This mechanism tailors the cutoff of documents to each query, reducing unnecessary budget consumption on irrelevant documents and preserving utility across multiple queries. 
We will see in the experimental section that this approach especially yields clear utility gains on datasets with high correlated queries. 
The full procedure is summarized in Algorithm~\ref{alg:batch_dp_rag-v2}. %

\begin{algorithm}[ht]
\caption{\dpadaptovetau: DP \textbf{Mu}lti-Query \textbf{RAG} with \textbf{Ada}ptive Threshold}\label{alg:batch_dp_rag-v2}
\setcounter{AlgoLine}{0}
\KwIn{Private dataset $D$, sequence of queries $\{q_1, \ldots, q_T\}$, per-query budget $\varepsilon_q$, number of retrieved documents $k$, maximum retrievals per document $M$} 

\KwSet{Initialize budget for each $z \in D$: $\mathcal{E}(z) \gets M \cdot \varepsilon_q$. 
Split budget: $\varepsilon_q = \varepsilon_{\mathrm{thr}} + \varepsilon_{\mathrm{RAG}}$.}

\KwRequire{Discretization of similarity scores into bins $[a_i,a_{i+1})_{i=1}^B$}

\For{$t = 1,\dots,T$}{
    \tcc{Step 1: Adaptive thresholding via noisy prefix sums}
    $\Tilde{s} \gets 0$, $A_t \gets \varnothing$\;
    
    \For{$i = 1,\ldots,B$}{
        $A_t^{(i)} = \{z \in D \mid r(z, q_t) \in [a_i, b_i], \, \mathcal{E}(z) \geq \varepsilon_{\mathrm{thr}}\}$\;
        
        $\Tilde{s} \gets \Tilde{s} + |A_t^{(i)}| + \mathrm{Lap}(1/\varepsilon_{\mathrm{thr}})$\;
        
        $A_t \gets A_t \cup A_t^{(i)}$\;
        
        \For{$z \in A_t^{(i)}$}{
            $\mathcal{E}(z) \gets \mathcal{E}(z) - \varepsilon_{\mathrm{thr}}$\;
        }
        
        \If{$\Tilde{s} \geq k$}{$\tau_t=a_i$; break \tcp*[r]{Release threshold}} 
    }
    
    \tcc{Step 2: DP-RAG on adaptively selected active set}
    $A_t^\prime = \{z \in A_t \mid \mathcal{E}(z) \geq \varepsilon_{\mathrm{RAG}}\}$\;
    
    $D_{q_t} = \topk(A_t^\prime, k, r(\cdot, q_t))$\;
    
    $a_t = \dprag(x, D_{q_t}, \llm, \varepsilon_{\mathrm{RAG}};\tau_t)$ \tcp*[r]{single-query RAG, Algorithm~\ref{alg:dp-rag-v2}}
    
    \For{$z \in A_t^\prime$}{
        $\mathcal{E}(z) \gets \mathcal{E}(z) - \varepsilon_{\mathrm{RAG}}$\;
    }
}
\KwRet{$(a_1, \ldots, a_T)$}
\end{algorithm}

Notice that in Algorithm~\ref{alg:batch_dp_rag-v2}, we use $k$ as a stopping criterion instead of releasing differentially private top-$k$ relevance scores. 
This is because releasing a noisy top-$k$ score for each query would make the privacy budget grow linearly with the number of queries and incur loss on all documents, thereby breaking the per-document privacy filter. By contrast, our prefix-sum approach (Step 1 of Algorithm~\ref{alg:batch_dp_rag-v2}) incurs privacy loss only on the documents that appear in the released prefix sums, while all other documents remain untouched. This concentrates the privacy cost of this step still on a small subset, yielding tighter accounting and more efficient budget use across multiple queries. The privacy guarantee of \dpadaptovetau{} is stated in Theorem~\ref{lem:privacy_batch_rag_ada}, and the proof is deferred to Appendix~\ref{apx:privacy_proof}.

\begin{theorem}[Privacy Guarantee of Algorithm~\ref{alg:batch_dp_rag-v2}]\label{lem:privacy_batch_rag_ada}
    $\dpadaptovetau$ satisfies $\varepsilon$-differential privacy provided that the initial privacy budget allocated to each document $z \in D$ is at most $\varepsilon$.
\end{theorem}

\section{Experiment}\label{sec: exp_res}
\begin{figure}[!t]
\centering
\begin{subfigure}[t]{0.3\textwidth}
    \includegraphics[width=\linewidth]{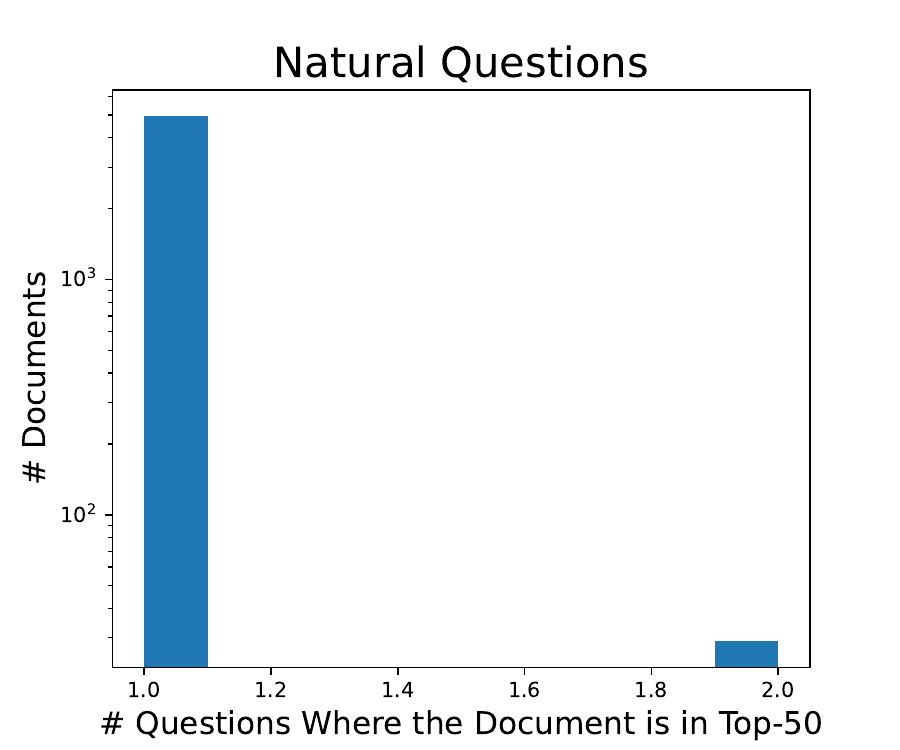}
\end{subfigure}
\begin{subfigure}[t]{0.3\textwidth}
    \includegraphics[width=\linewidth]{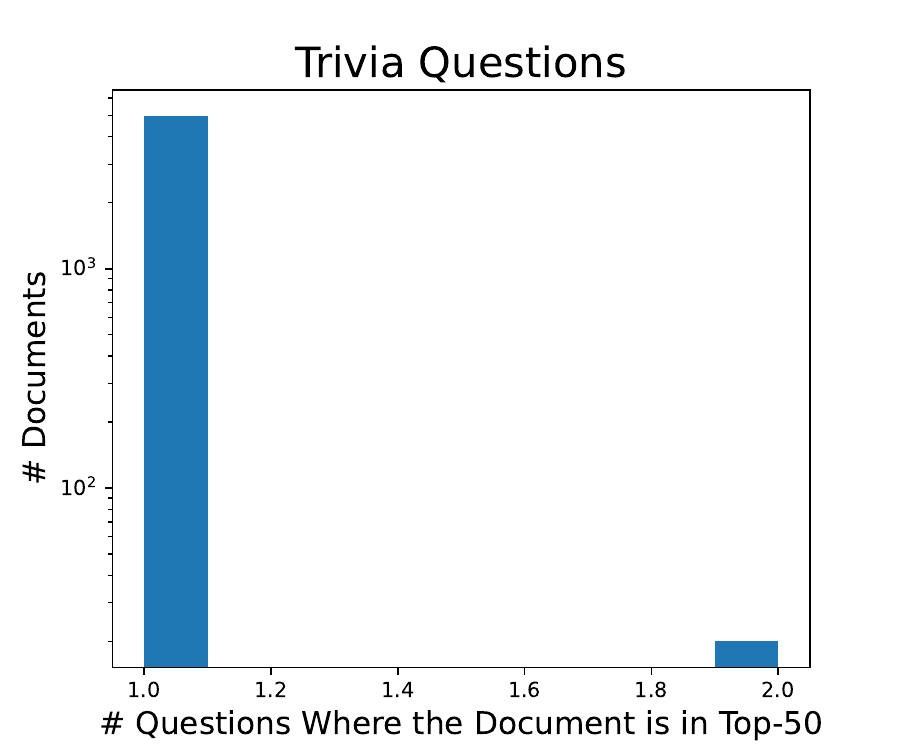}
    \end{subfigure}
\begin{subfigure}[t]{0.3\textwidth}

    \includegraphics[width=\linewidth]{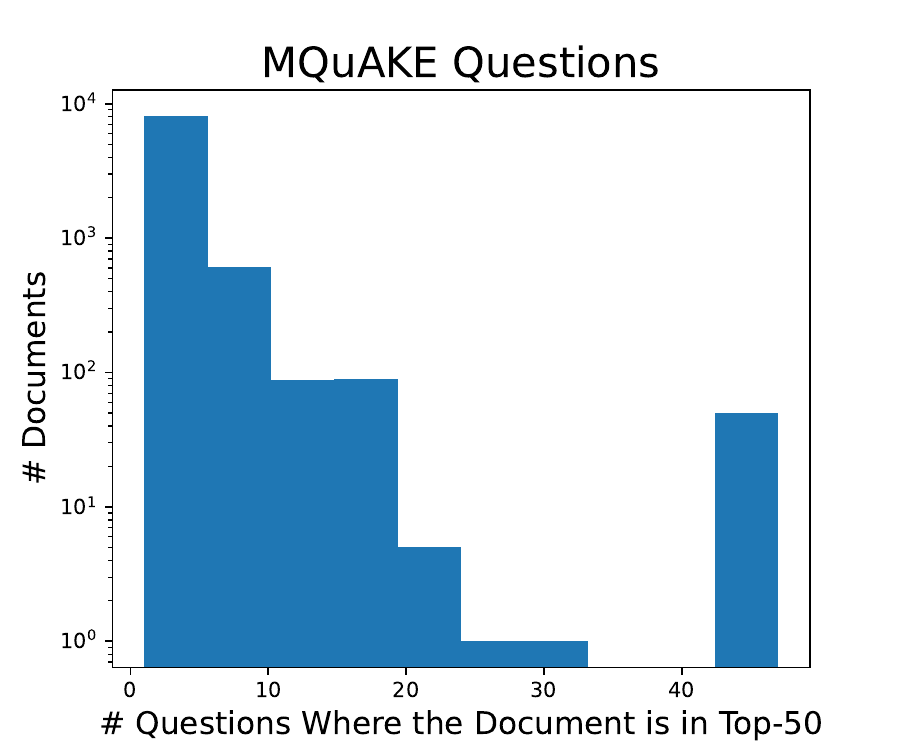}
\end{subfigure}
\captionsetup{font=small}
\caption{Histogram of document reuse across questions. Each bar shows how many questions a document appears in among the top-$K$ retrieved results ($K=50$). The x-axis indicates the number of questions per document, and the y-axis shows the count of such documents. 
}
\vspace{-2ex}
\label{fig:hist}
\end{figure}

\subsection{Dataset}
\textbf{Datasets set-up.}
We first evaluate our methods on \textbf{two independent question sets}: \textit{Natural Questions} and \textit{Trivia Questions}. These are standard benchmarks for evaluating RAG systems and have been used in prior work on per-query DP for RAG~\citep{koga2024privacy}. Following their setup, we randomly subsample 100 questions from each dataset to reduce computational overhead. Importantly, the questions are independent of one another, and each requires a disjoint set of relevant documents from the external database. To quantify document reuse, we examine how frequently each document appears in the top-$K$ retrieved results ($K=50$) across questions. As shown in Figure~\ref{fig:hist}, in both \textit{Natural Questions} and \textit{Trivia Questions}, most documents are retrieved for only one or two queries. \textit{Thus, we expect \dpfixtau{} to perform sufficiently well on these two datasets.}

Second, we consider a \textbf{correlated question set}, \emph{MQuAKE}~\citep{zhong2023mquake}. This dataset contains sequences of semantically related single-hop questions that together form multi-hop reasoning chains. We select 100 such sequences, yielding 400 individual questions for evaluation. Since questions in the same sequence share entities (subjects or objects), their relevant documents substantially overlap. As shown in Figure~\ref{fig:hist}, many documents appear across multiple questions. \textit{We therefore expect \dpadaptovetau{} to have an advantage over \dpfixtau{}.}

Finally, we evaluate on \emph{ChatDoctor}~\citep{li2023chatdoctor}, a \textbf{privacy-sensitive application of RAG} in the healthcare domain. This dataset consists of QA interactions between patients and doctors. We sample 100 patient questions as our test set. \textit{This evaluation tests the effectiveness of our methods in a real-world sensitive setting and their robustness against privacy attacks.}

\textbf{External datasets reflecting both standard and privacy-sensitive settings.}
For Natural Questions, Trivia Questions, and MQuAKE Questions, we use Wikipedia of $\sim20M$ documents as the external knowledge source following the standard RAG setup~\citep{chen2017reading, lewis2020retrieval}.
For ChatDoctor Questions, the external dataset consists of the remaining $\sim200K$ QA pairs from the original ChatDoctor dataset, excluding the 100 patient questions used for testing. This setup reflects a realistic privacy-sensitive application, where the external corpus contains private information.

\textbf{QA evaluation metric.} For Natural Questions, Trivia Questions and MQuAKE Questions, the datasets provide a list of all acceptable correct answers for each question. Following the evaluation protocol of \citet{koga2024privacy}, we use the \textit{Match Accuracy} metric: a prediction is scored as 1 if it contains any correct answer, and 0 otherwise.  For Chatdoctor Questions, we adopt the evaluation metric from the original dataset paper, using the F1 score of BERTScore~\citep{bert-score} to measure semantic similarity between the predicted response and the ground-truth answer.
\vspace{-1em}

\subsection{Model and Method set-up}
\textbf{Model set-up.} Our RAG pipeline integrates three pre-trained LLMs: OPT-1.3B~\citep{zhang2022opt}, Pythia-1.4B~\citep{biderman2023pythia}, and Mistral-7B~\citep{jiang2023mistral7b}. For document retrieval, we use the Dense Passage Retriever (DPR)~\citep{karpukhin2020dense} to compute dense query-document relevance scores.

\textbf{Baseline methods.} We compare our two proposed methods with five baselines. The first is \textbf{$\naivedprag$} (Algorithm~\ref{alg:naive_batch_dp_rag}), which applies the per-question DP RAG method, DPSparseVoteRAG, independently to each query and uses the standard sequential composition theorem~\citep{dwork2006our} to compute the overall privacy guarantee.
The second baseline applies subsampling amplification to the first baseline, \naivedprag, which we called \textbf{\sampledprag}. Specifically, for each query, we first subsample the external dataset using Poisson sampling with rate $\eta$, and then apply DPSparseVoteRAG on this subsampled dataset. The overall privacy guarantee is then computed using sequential composition combined with the amplification by subsampling \citep{balle2018privacy}.
The third baseline privatizes the external dataset of RAG under differential privacy (DP) and then uses the resulting synthetic dataset as the knowledge source for evaluation. In this setup, the answers are guaranteed to satisfy DP since they are derived from a privatized dataset. We adopt \textbf{Private Evolution} (PE; \citet{xie2024differentially}), a state-of-the-art DP synthetic text generation method that also aligns with the query-access setting of RAG. Specifically, PE first queries an LLM to produce an initial dataset within the same domain as the private corpus, and then refines its distribution under DP to better approximate that of the private dataset. To ensure consistency, for each pretrained LLM used in RAG, we use the same model as the query API in PE.
The other two are non-private baselines: \textbf{Non-RAG}, which generates answers using the pretrained LLM without retrieval, and \textbf{Non-Private-RAG}, which performs retrieval-augmented generation without any privacy mechanism. We describe implementation details in Appendix~\ref{app:exp}.

\textbf{Privacy budget setup for DP algorithms.} 
Following the setup in \citet{koga2024privacy}, we vary the per-query RAG privacy budget $\varepsilon_{q} \in \{2, 5, 10, 15, 20, 30, 40\}$ to explore the privacy-utility trade-off. For $\naivedprag$, the total privacy budget is $T \cdot \varepsilon_{q}$, where $T$ is the number of questions. For \dpfixtau\ and \dpadaptovetau, the total budget is $M \cdot \varepsilon_{q}$, where $M$ is the number of retrieved documents with nonzero privacy loss\footnote{ To enable a meaningful comparison, we convert our privacy guarantee, originally expressed in $(\infty,\varepsilon)$-RDP, into an equivalent $\varepsilon$-DP guarantee \citep{mironov2017renyi}. }. In our main results, we conservatively set $M = 1$ for a realistic privacy region in \dpfixtau{} and \dpadaptovetau{} and set $\varepsilon_{\rm thr}$ as $1.0$ in \dpadaptovetau\footnote{We will see the detailed analysis of the choices of $M$ and $\varepsilon_{\rm thr}$ in Section~\ref{sec:exp_further_analysis}}.
For the baseline \sampledprag, we consider the subsampling rate $\eta=0.1, 0.01, 0.001$ and calculate the corresponding $\varepsilon_q$ to satisfy the the varying total budget $\{2, 5, 10, 15, 20, 30, 40\}$.
For the baseline PE, we test with $\varepsilon\in\{10, 200\}$.

\textbf{Membership inference attack in RAG.} To assess the effectiveness of our privacy-preserving methods, we evaluate them against the membership inference attack (MIA). 
The objective of MIA is as follows: given a candidate document $x$ and a model system $R(\cdot; D)$ trained on a private dataset $D$, the adversary aims to determine whether $x \in D$ by computing a membership score $s(x, R(\cdot; D))$. Without loss of generality, we assume higher scores indicate higher membership likelihood. Applying the attack to an in-distribution set $D_{\text{in}} \subset D$ and an out-of-distribution set $D_{\text{out}}$ (with no overlap with $D$) allows us to derive the TPR–FPR curve and compute the AUC, which serves as the evaluation metric for attack success. 

We focus on scenarios where the adversary can issue multiple queries to the system, as this setting substantially amplifies the attack strength. To model this, we adopt the \textit{Interrogation Attack (IA)}~\citep{naseh2025riddle}, a state-of-the-art MIA specifically designed to exploit multi-query access in RAG systems. For each document $x$, IA generates $m=30$ tailored questions together with their corresponding answers implied by $x$. Then each question is concatenated with the necessary context to ensure the target document can be retrieved, and the query is then submitted to the RAG system. The membership score is defined as the accuracy of the RAG system across these $m$ questions, where higher accuracy implies a greater likelihood that the document is present in the external dataset and is being retrieved to answer the queries. Additional implementation details, including the question generation process, are provided in Appendix~\ref{app:exp}.

\subsection{Main Results}

\begin{figure}[!t]
\centering
\begin{subfigure}[t]{0.32\textwidth}
    \includegraphics[width=\linewidth]{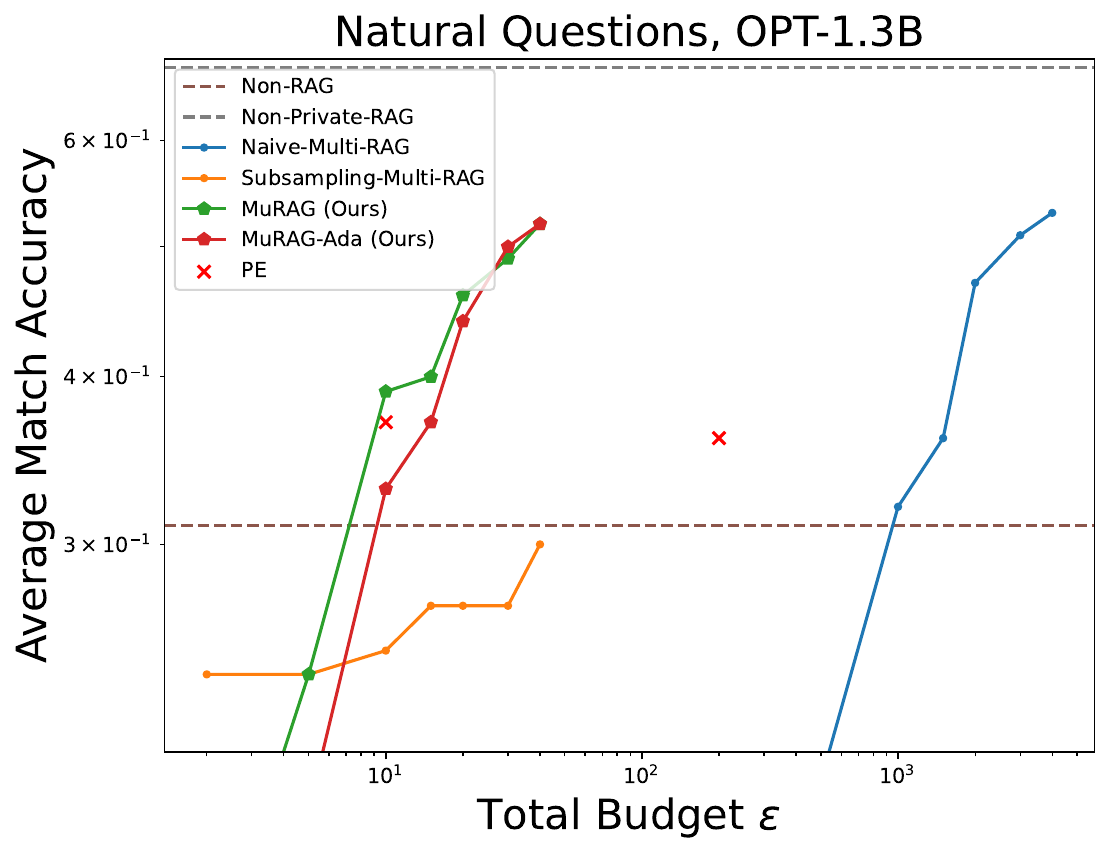}
\end{subfigure}
\hfill
\begin{subfigure}[t]{0.32\textwidth}
    \includegraphics[width=\linewidth]{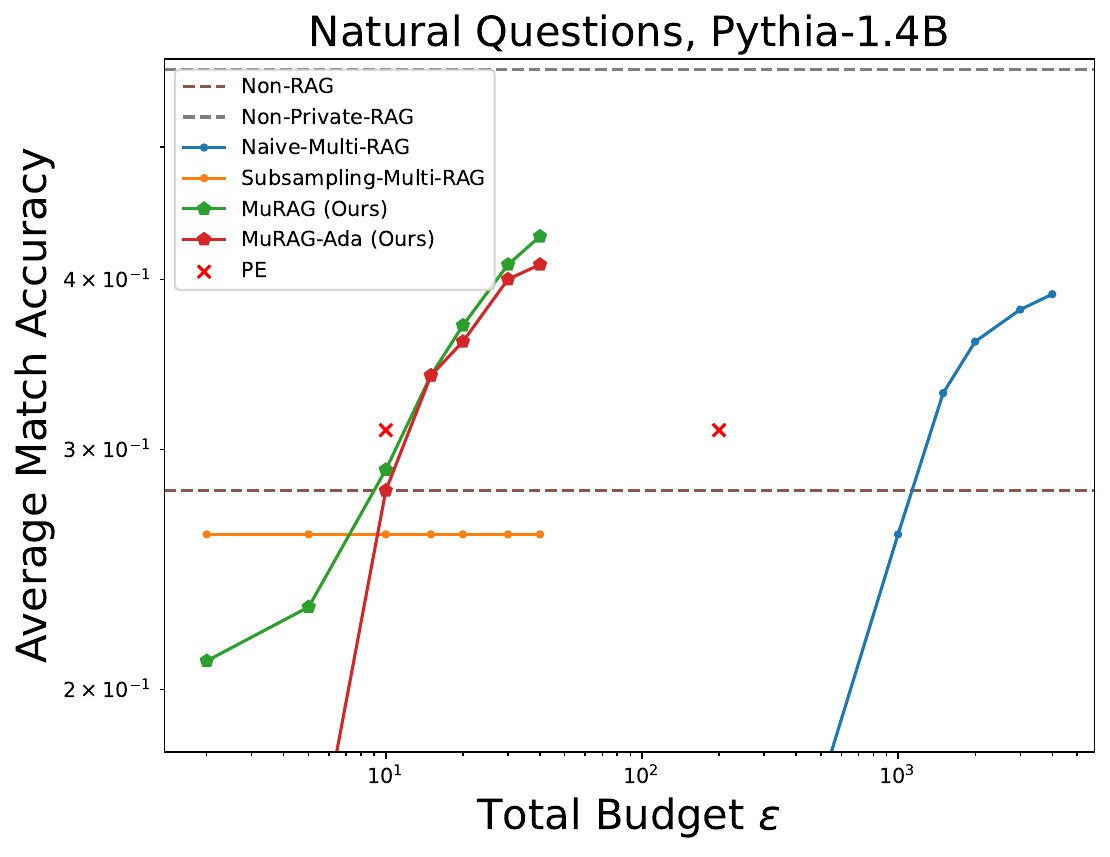}
\end{subfigure}
\hfill
\begin{subfigure}[t]{0.32\textwidth}
    \includegraphics[width=\linewidth]{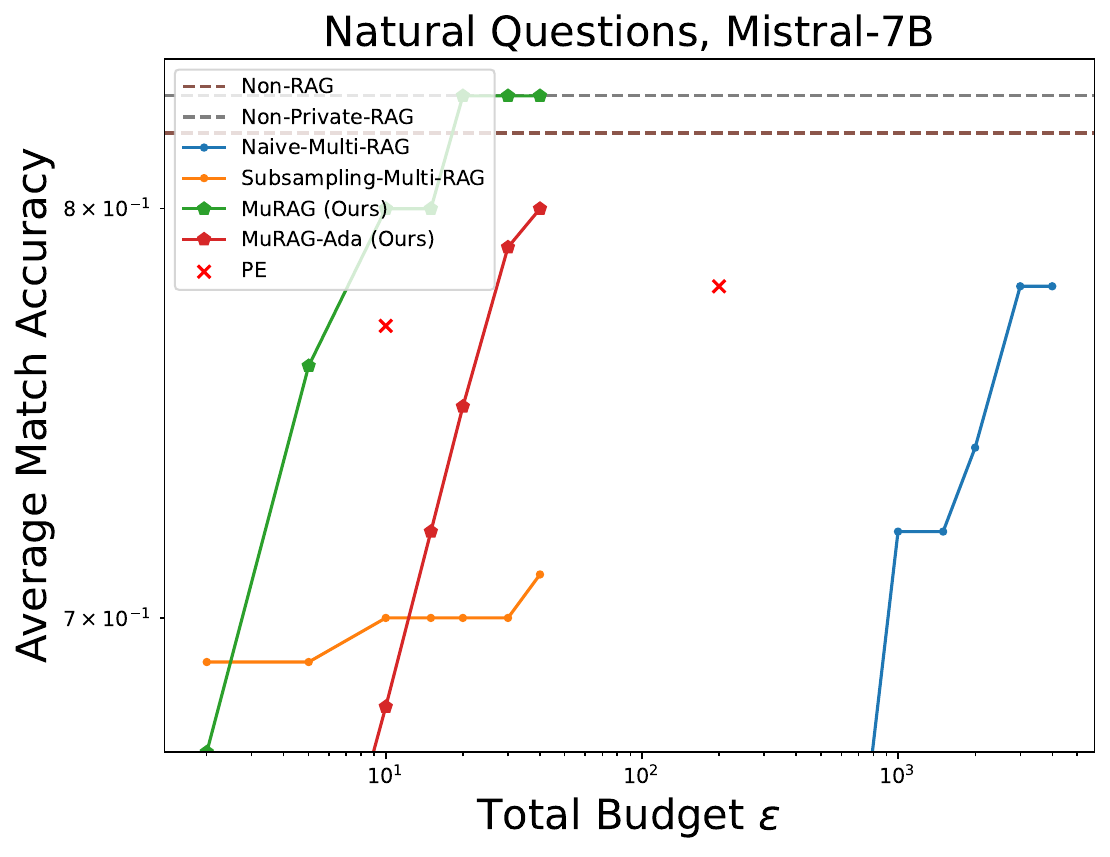}
\end{subfigure}

\begin{subfigure}[t]{0.32\textwidth}
    \includegraphics[width=\linewidth]{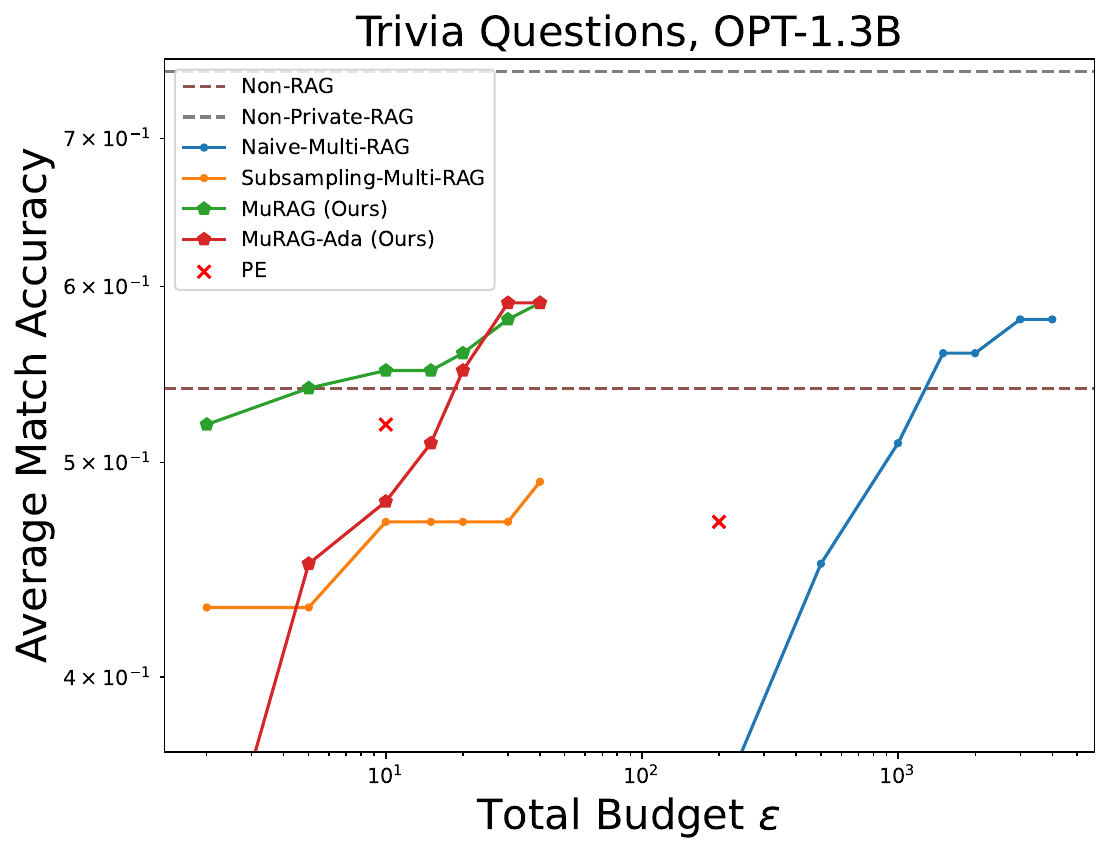}
\end{subfigure}
\hfill
\begin{subfigure}[t]{0.32\textwidth}
    \includegraphics[width=\linewidth]{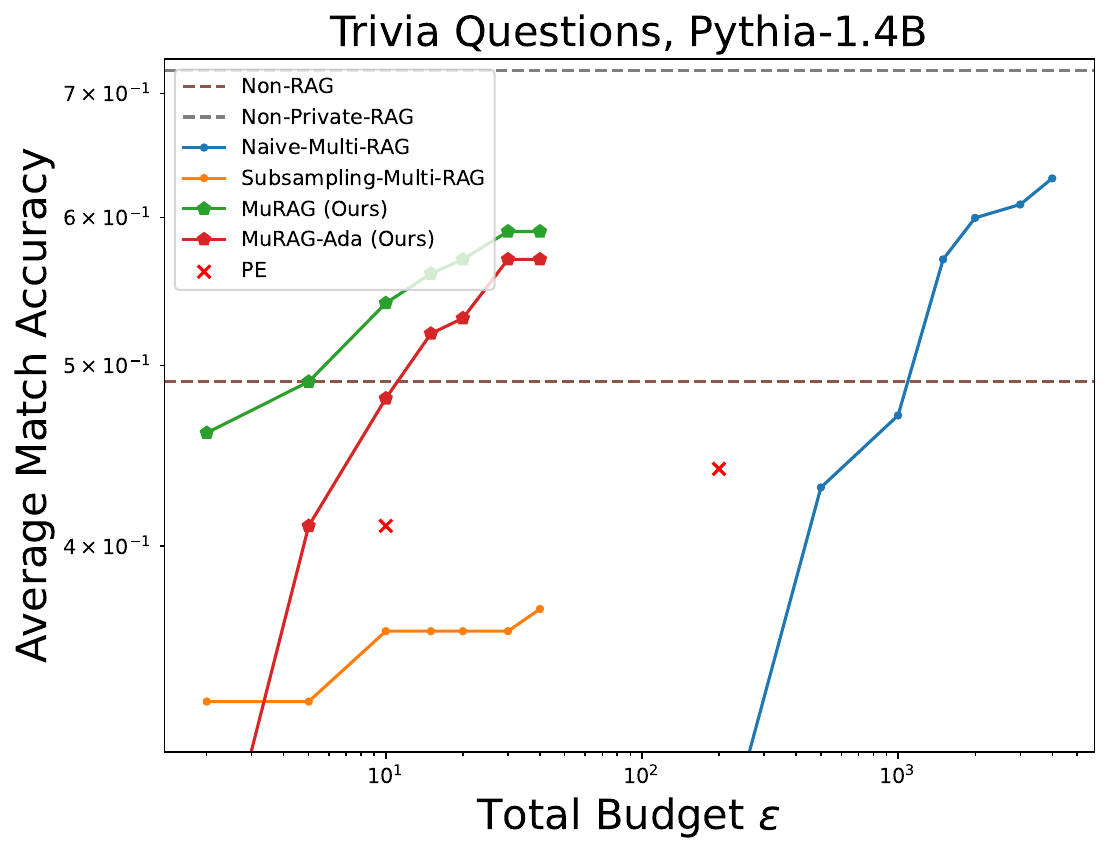}

\end{subfigure}
\hfill
\begin{subfigure}[t]{0.32\textwidth}
    \includegraphics[width=\linewidth]{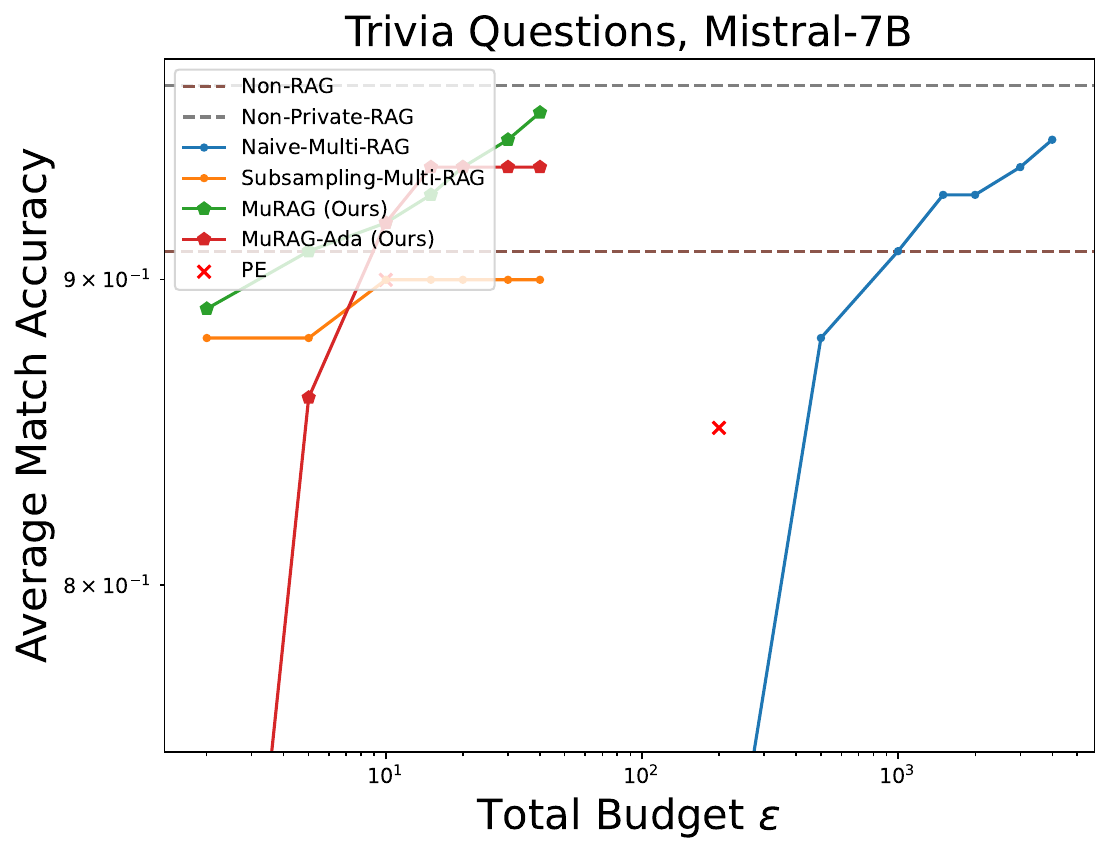}
\end{subfigure}

\begin{subfigure}[t]{0.32\textwidth}
    \includegraphics[width=\linewidth]{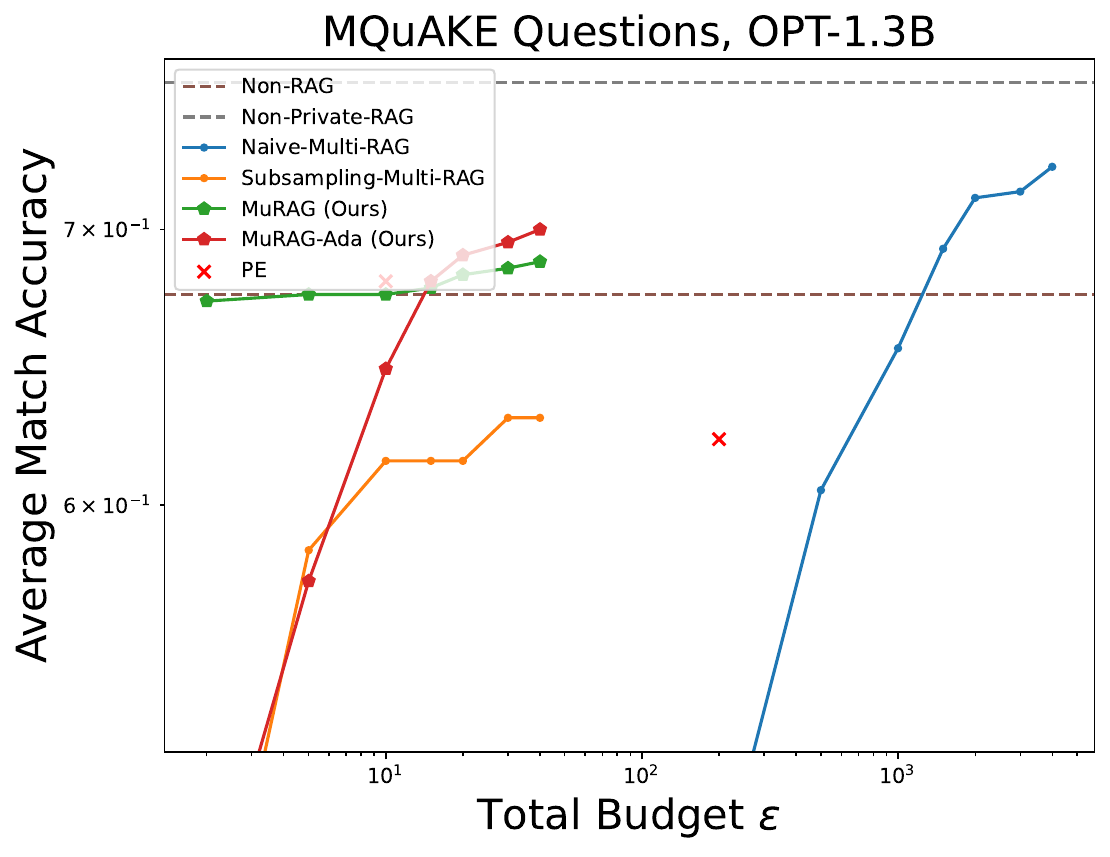}
\end{subfigure}
\hfill
\begin{subfigure}[t]{0.32\textwidth}
    \includegraphics[width=\linewidth]{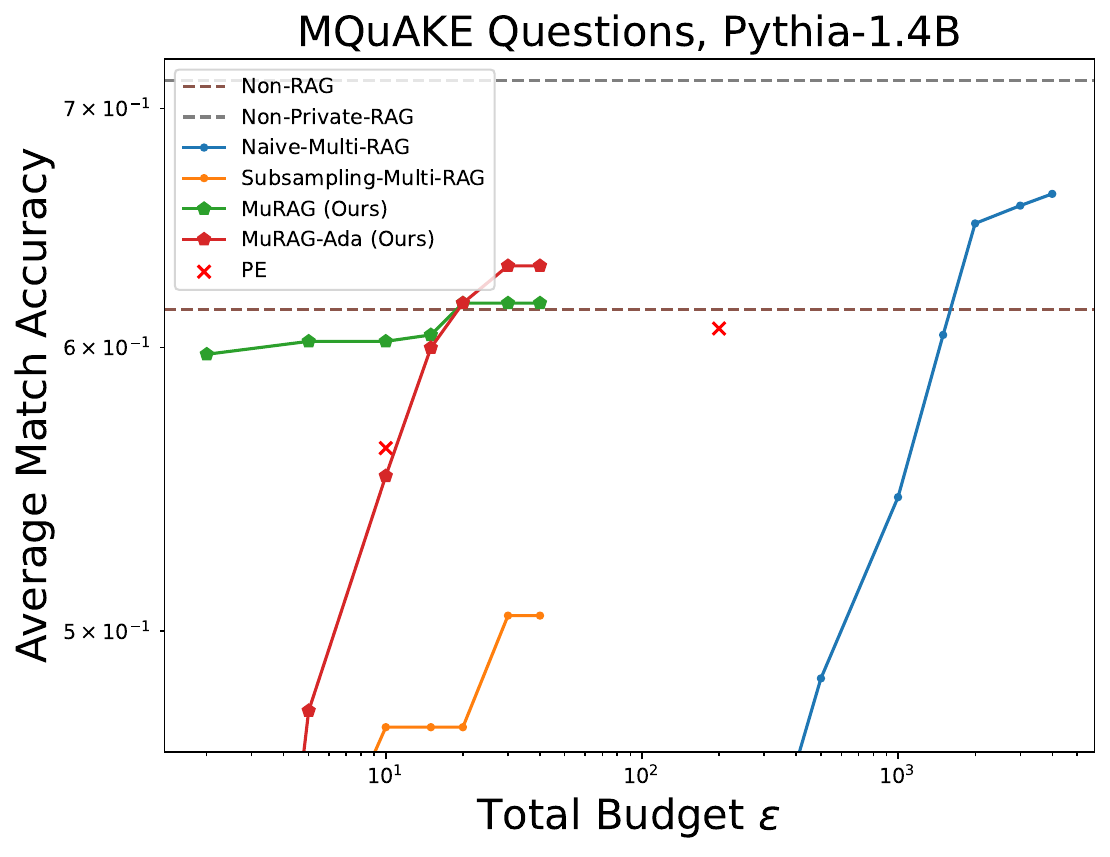}
\end{subfigure}
\hfill
\begin{subfigure}[t]{0.32\textwidth}
    \includegraphics[width=\linewidth]{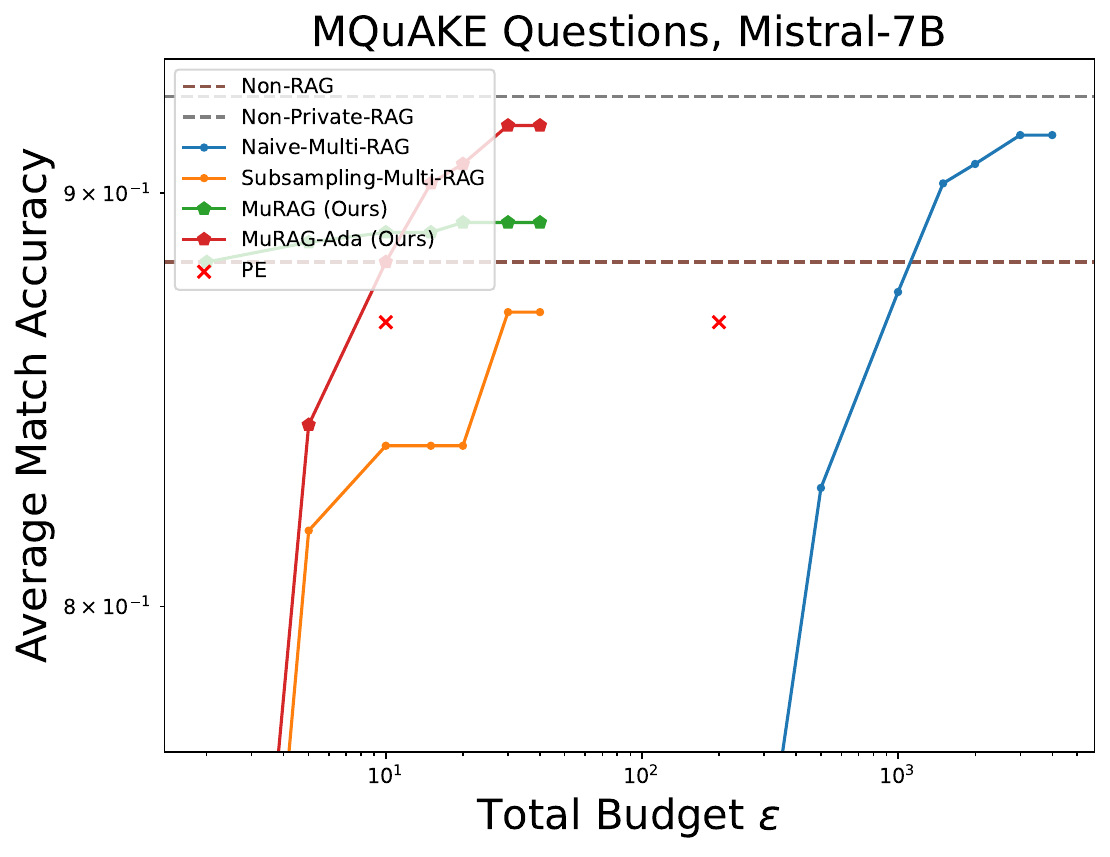}
\end{subfigure}
\captionsetup{font=small}
\caption{Privacy-Utility tradeoffs of our two proposed methods (\dpfixtau\ and \dpadaptovetau) compared to baselines  across three pretrained LLMs and two categories of question sets.
}
\label{fig:main}
\end{figure}
\vspace{-1em}

\textbf{Results on two standard RAG benchmarks (independent question sets).}
Figure~\ref{fig:main} shows the performance of our two proposed methods compared with three baselines across three pretrained LLMs on \textit{Natural Questions} and \textit{Trivia Questions}. Both of our methods outperform the Non-RAG baseline in most cases under a total privacy budget of $\varepsilon = 10$.

In contrast, all DP baselines (\textsc{Naive-Multi-RAG}, \sampledprag{}, PE) either underperform the Non-RAG model or require an impractically large privacy budget to achieve comparable performance.
The baseline \textsc{Naive-Multi-RAG} requires an impractically large budget, exceeding $\varepsilon = 10^3$, to achieve comparable utility. This highlights that our approaches make differential privacy practical in the multi-query RAG setting by leveraging more tailored compositions, enabling strong utility within a realistic privacy budget. The \sampledprag{} baseline consistently underperforms the Non-RAG model. This degradation is likely due to the reduced number of effective documents (that provide the ground truth answers) after subsampling. For example, if there are $50$ relevant documents for a query, subsampling at a rate of $0.1$ leaves only about $5$ accessible documents, making it difficult for \textsc{DPSparseVoteRAG} to produce correct answers within the per-query budget $\varepsilon_q\approx 0.71$ (computed from the overall budget $\varepsilon$, total queries $T=100$, and sampling rate $\eta=0.1$). The results demonstrate that the individual privacy accounting framework provides a more effective composition mechanism than subsampling amplification for multi-query RAG problem; a more detailed discussion of this limitation is provided in Section~\ref{sec:discuss}.
The \textsc{PE} baseline performs even worse than Non-RAG at $\varepsilon=200$ for many settings, which we attribute to objective misalignment: PE optimizes for distributional similarity (e.g., measured by Fréchet Inception Distance (FID; ~\citet{heusel2017gans})) rather than preserving factual content. Indeed, we find PE achieves a better FID score at $\varepsilon=200$ but yields lower task performance than at $\varepsilon=10$ on the setting of Trivia Questions and OPT-1.3B, further supporting this explanation.\footnote{We confirm that the FID score improves from $\varepsilon=10$ to $\varepsilon=200$ ($0.066$ to $0.036$; lower is better) on the setting of Trivia Questions and OPT-1.3B, yet RAG utility drops, underscoring the mismatch between FID and factual fidelity required for RAG.}

Lastly, on these two datasets, \dpfixtau{} outperforms \dpadaptovetau{}, which aligns with our expectations. Since the questions are independent, adaptive thresholding provides little benefit and additionally consumes extra privacy budget.

\textbf{Results on multi-hop questions (correlated question set).}
Figure~\ref{fig:main} shows the performance of our two proposed methods compared with three baselines across three pretrained LLMs on \textit{MQuAKE Questions}. Overall, the relative trends between our methods and the baselines are consistent with the independent question setting. However, a key difference emerges in the comparison between our two approaches: \dpadaptovetau{} performs significantly better than \dpfixtau{}. This result is aligned with our intuition, as adaptive thresholding is particularly advantageous when questions are correlated and share overlapping relevant documents.

\begin{figure}
    \centering
    \includegraphics[width=0.8\textwidth]{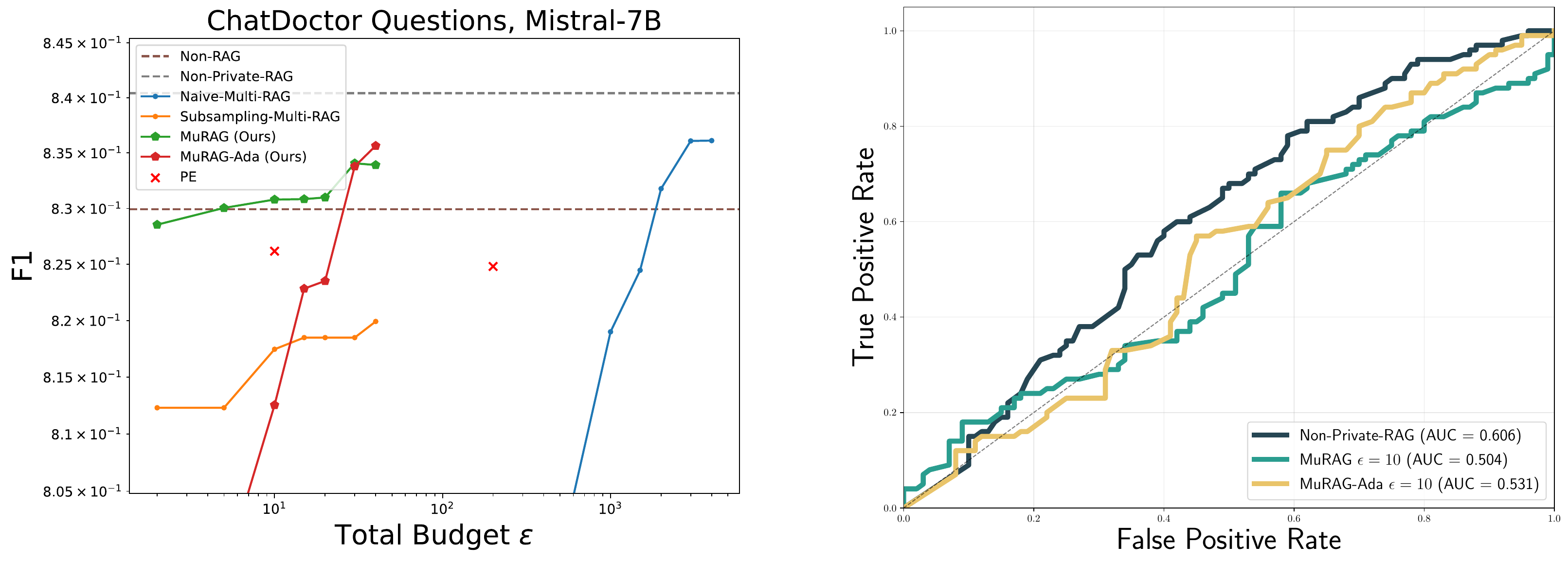}
    \captionsetup{font=small}
    \caption{\textbf{Left:} Privacy-utility tradeoffs of our two methods and baselines. \textbf{Right:} TPR-FPR curves of \textsc{IA} (Membership Inference Attack with multiple queries). Both experiments are conducted with Mistral-7B and ChatDoctor datasets.}
    \label{fig:chatdoctor}
\end{figure}

\textbf{Results on privacy-sensitive application.} The left plot in Figure~\ref{fig:chatdoctor} shows the performance of our methods and baselines on Mistral-7B with the ChatDoctor dataset. The results mirror the trends observed in the previous benchmarks: both of our methods outperform the baselines in this practical, privacy-sensitive setting. In particular, \dpfixtau{} surpasses the Non-RAG baseline at $\varepsilon=10$.%

We also evaluate robustness against the Interrogation Attack (IA) on ChatDoctor. Specifically, we test three RAG systems: Non-Private-RAG, \dpfixtau{} ($\varepsilon=10$), and \dpadaptovetau{} ($\varepsilon=10$). The right plot in Figure~\ref{fig:chatdoctor} reports the corresponding TPR–FPR curves. Without protection, \textsc{IA} achieves a non-trivial AUC of $\approx 0.6$. In contrast, both of our methods reduce the AUC to $\approx 0.5$, making the attack ineffective. These findings demonstrate that our approaches provide practical privacy protection at $\varepsilon=10$ in a real-world sensitive application.

\textbf{Takeaway.} Across all evaluations, our methods consistently outperform baseline approaches under practical privacy budgets. On independent question sets, \dpfixtau{} achieves strong performance as expected, while on correlated multi-hop questions, \dpadaptovetau{} shows clear advantages due to its adaptive thresholding. Finally, in the privacy-sensitive ChatDoctor application, both methods not only improve utility over baselines but also effectively mitigate state-of-the-art membership inference attacks. Together, these results demonstrate that our approaches make differentially private RAG both practical and robust across diverse settings.
\vspace{-0.6em}

\subsection{Further Analysis of $\dpfixtau$ and \dpadaptovetau}
\label{sec:exp_further_analysis}
\textbf{Comparison between thresholding approaches in our two methods.} The two methods have different performance as discussed above, and the difference is between the constant thresholding and the DP-released adaptive thresholding.
To quantify this effect, Table~\ref{tab:precision_of_retrival} reports the precision under both \textit{constant thresholds} (in \dpfixtau) and \textit{adaptive thresholds} (in \dpadaptovetau), where we measure the percentage of truly top-50 documents among the retrieved documents for each question and calculate the average over questions as the precision. We observe that precision under \dpfixtau\ is particularly low for the correlated question set \textit{MQuAKE Questions}, whereas \dpadaptovetau\ significantly improves retrieval precision on these datasets through its adaptive thresholds. This improvement in retrieval quality directly contributes to the superior performance of \dpadaptovetau\ in the setting of \textit{correlated question set}.

\begin{table}
\centering
\caption{Precision of retrieved documents under different thresholding approaches, measured as the percentage of truly top-50 relevant documents among the retrieved. 
}
\label{tab:precision_of_retrival}
\vspace{0.4em}
\resizebox{\textwidth}{!}{%
	\begin{tabular}{c|cccc}
	\toprule
		& \multicolumn{2}{c}{\textbf{Independent Question Set}} & \textbf{Correlated Question Set} \\
		& Natural Questions & Trivia Questions & MQuAKE Questions \\
		\midrule
		Constant Thresholding (in \dpfixtau) & 78.8\%  & 72.2\%  & 17.6\% \\
		Adaptive Thresholding (in \dpadaptovetau{}) & 92.6\%  & 94.6\%  & 40.7\% \\
            Adaptive Thresholding (Non-private top-K-release) & 99.4\%  & 99.6\%  & 43.5\% \\
		\bottomrule
	\end{tabular}
}
\end{table}

\textbf{Effect of different $M$ in the individual privacy accounting framework.} 
Both of our proposed methods include a hyperparameter $M$, which controls the maximum number of queries for which an individual document's privacy budget can be consumed. In our main results (Figure~\ref{fig:main}), we set $M = 1$ to ensure strict per-document privacy usage. However, this setting may limit utility: once a document is used for one query, it becomes unavailable for future queries, even if it would have been highly relevant. 
To better understand the impact of $M$, we evaluate our two methods with a larger value of $M = 5$.
The left plot in Figure~\ref{fig:M_study} shows a substantial increase in Top-50 retrieval precision when using $M=5$, indicating better access to relevant documents. This improvement translates into higher end-to-end RAG utility, as shown in the three plots on the right. However, increasing $M$ also leads to a higher total privacy cost ($\varepsilon_{\rm total}= M\cdot \varepsilon_{q}$). 

\begin{figure}[!t]
\centering
\begin{subfigure}[t]{0.24\textwidth}
    \includegraphics[width=\linewidth]{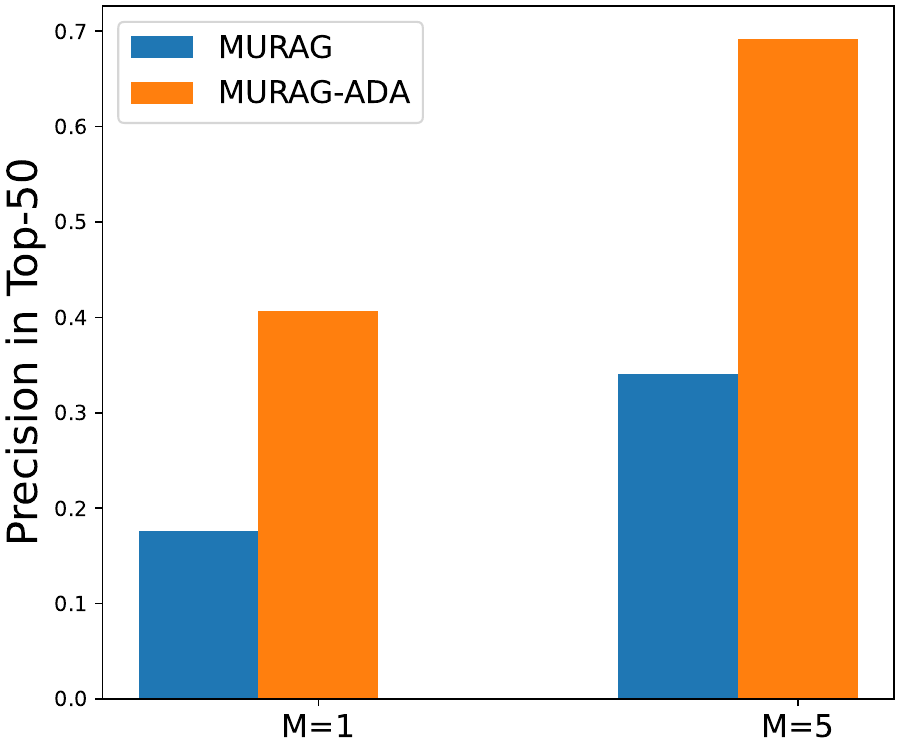}
\end{subfigure}
\hfill
\begin{subfigure}[t]{0.24\textwidth}
    \includegraphics[width=\linewidth]{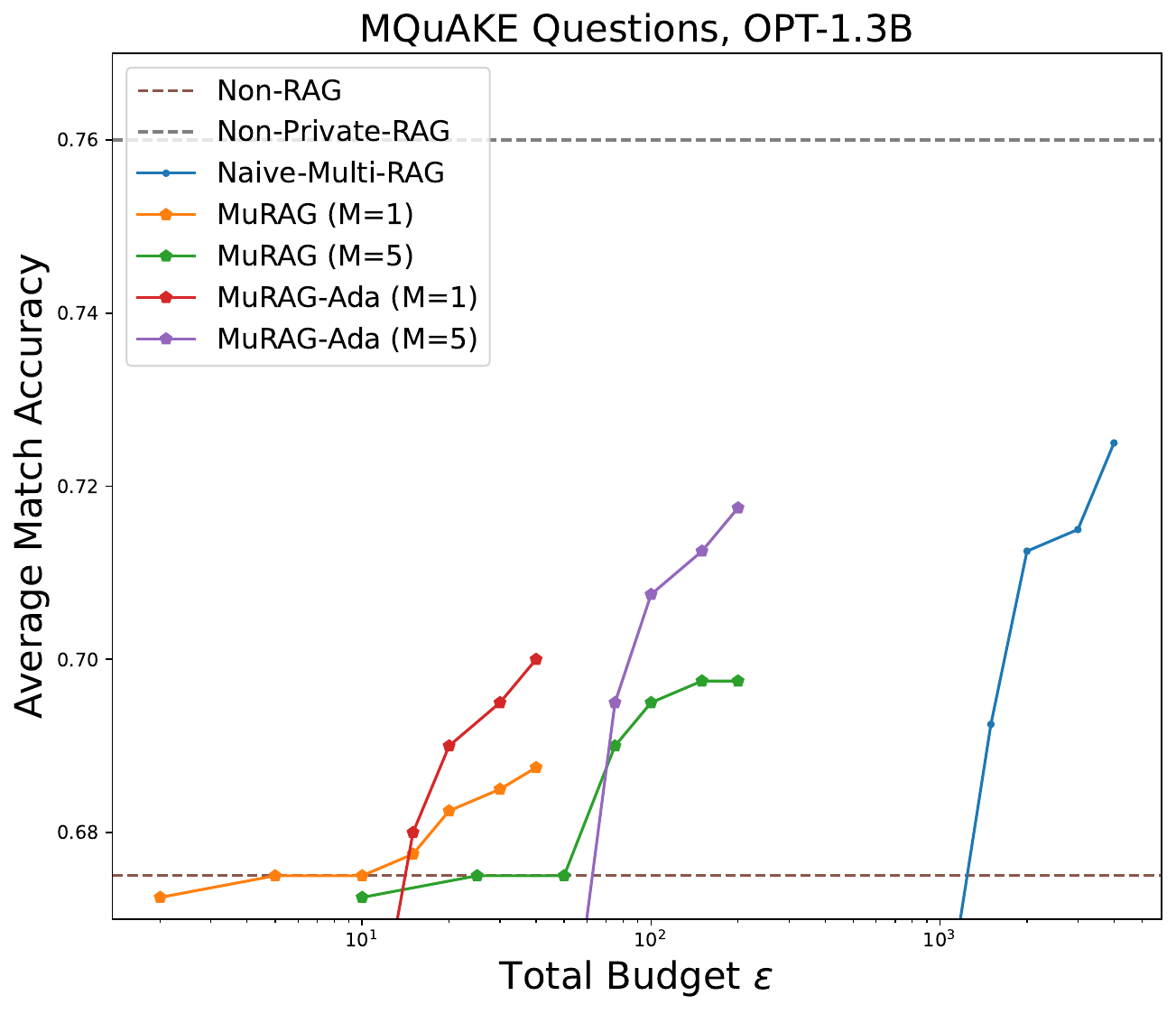}
    \end{subfigure}
\hfill
\begin{subfigure}[t]{0.24\textwidth}
    \includegraphics[width=\linewidth]{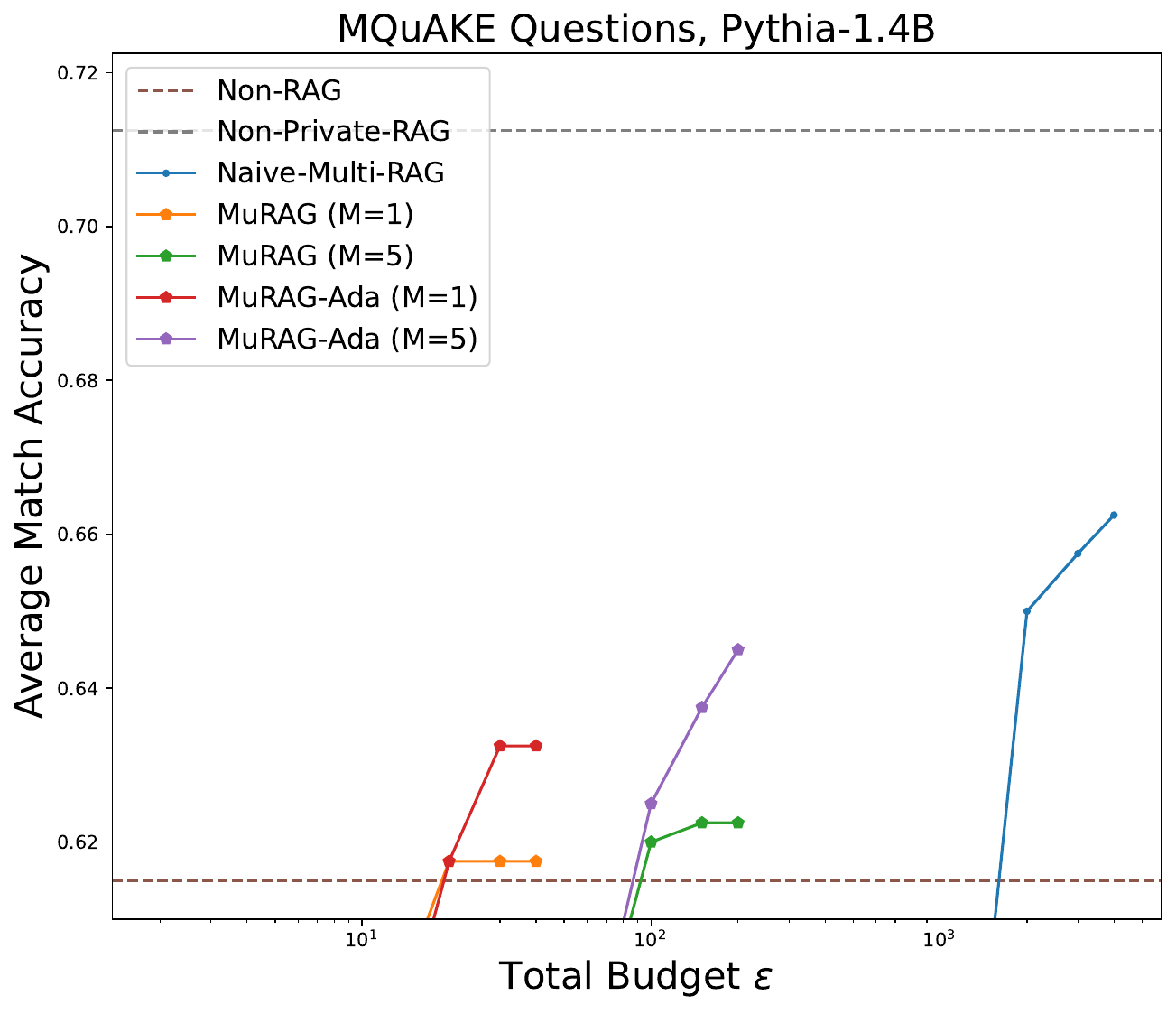}
\end{subfigure}
\hfill
\begin{subfigure}[t]{0.24\textwidth}
    \includegraphics[width=\linewidth]{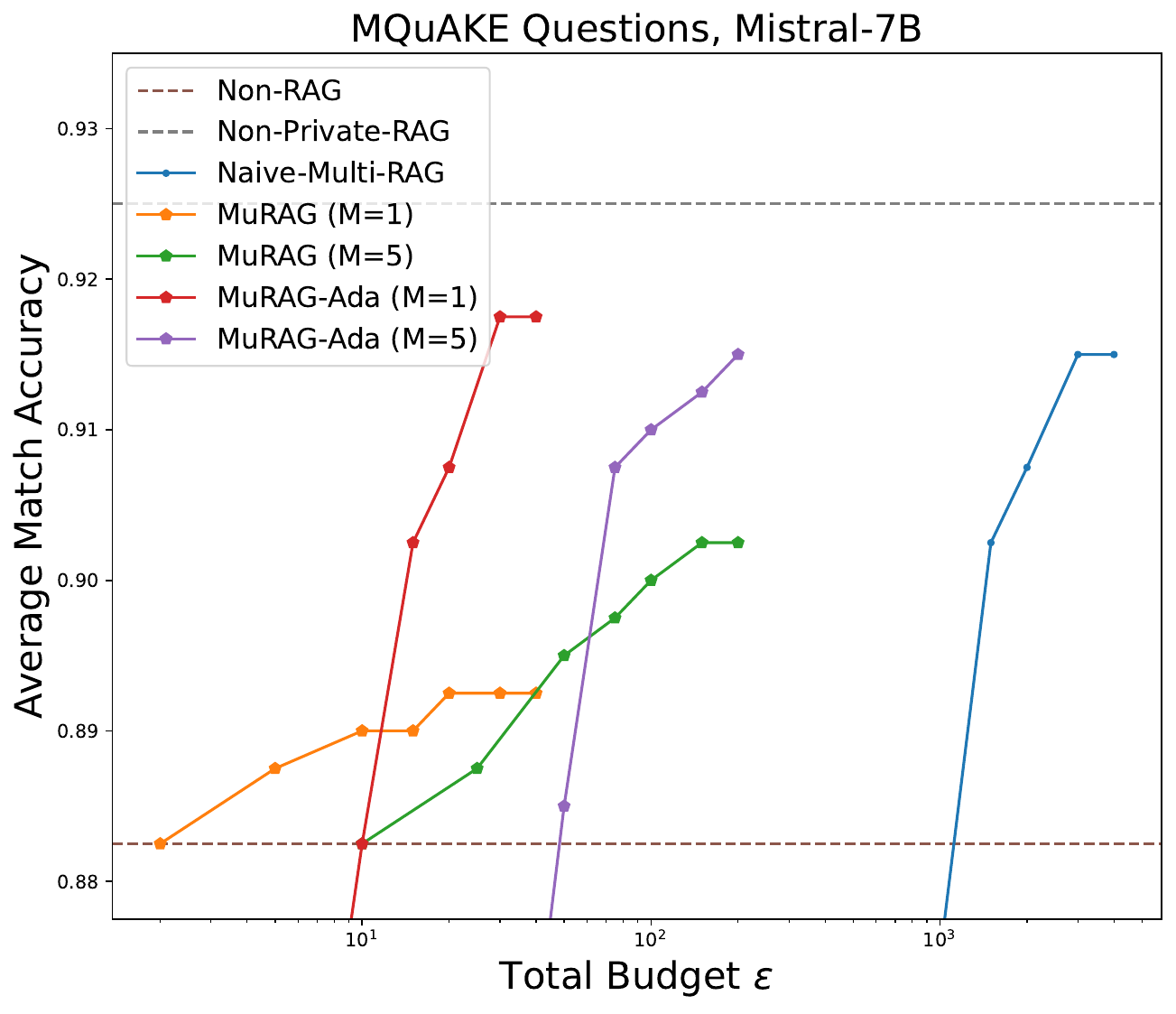}
\end{subfigure}
\captionsetup{font=small}
\caption{Comparison of $M=1$ and $M=5$ in the individual privacy accounting framework. The left plot shows the retrieval precisions of two methods with $M=1,5$. Right three plots show the trade-off between the QA performance and the $\varepsilon_{\rm total}$ in DP.
}
\label{fig:M_study}
\end{figure}

\begin{figure}[t]
\centering
\begin{minipage}{0.33\linewidth}
    \centering
    \includegraphics[width=\linewidth]{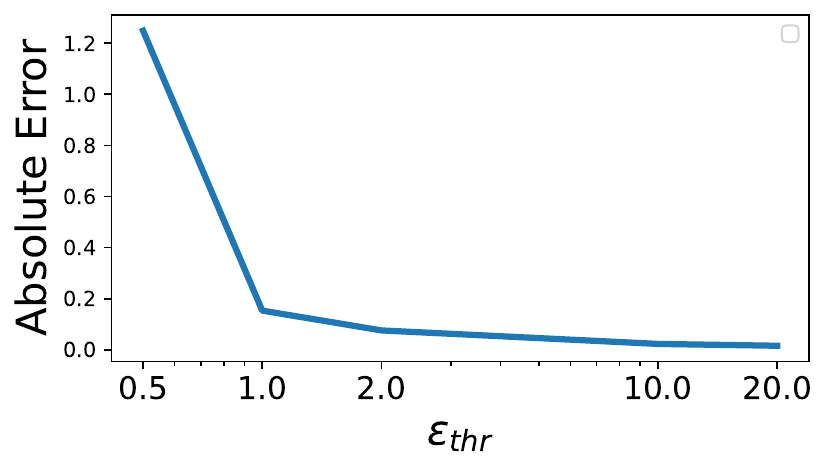}
    \caption{Absolute error of releasing $\tau_t$ in \dpadaptovetau.}
    \label{fig:tau_release}
\end{minipage}\hfill
\begin{minipage}{0.63\linewidth}
\textbf{Budget allocation in \dpadaptovetau.} An important hyperparameter, $\varepsilon_{\mathrm{thr}}$, controls the privacy budget allocated for releasing the threshold $\tau_t$.
Figure~\ref{fig:tau_release} shows the absolute error between the true top-$K$ threshold and the estimated threshold returned by the DP threshold-release procedure (Lines 3–10 in Algorithm~\ref{alg:batch_dp_rag-v2}) on the \textit{Trivia Questions} dataset. As shown, the estimation error remains small (absolute error $\leq 0.2$) when $\varepsilon_{\mathrm{thr}} \geq 1.0$, which is quite reasonable given that most scores lie between 70 and 100.
Based on this trade-off, we choose $\varepsilon_{\mathrm{thr}}$ so that it achieves a small absolute error while consuming only a small fraction of the total budget $\varepsilon$, leaving the remaining budget for the DP-RAG token-generation steps.

\end{minipage}
\end{figure}

\BlankLine\BlankLine
\section{Related Work}

Recent studies identify two main privacy risks in retrieval-augmented generation (RAG) systems. 
The first is membership inference attacks (MIA)~\citep{shokri2017membership}, which test whether a specific document is in the private external dataset, often via adversarial prompts~\citep{naseh2025riddle,liu2025mask,anderson2024my} or scoring mechanisms~\citep{li2025generating}. 
The second is data reconstruction attacks, which aim to recover document content using adversarial prompts~\citep{zhang2025deal,zeng2024good,jiang2024rag} or poisoning triggers~\citep{peng2024data}. 
Together, these works highlight the growing need for principled privacy-preserving algorithms for RAG.

Several DP-based defenses have been proposed. 
\citet{koga2024privacy} introduced a single-query DP-RAG system, and others~\citep{yao2025private,grislain2025rag} studied DP release of document identifiers. 
However, none of these methods address the realistic multi-query setting. 
In addition to DP based methods, empirical defenses have also been explored, including paraphrasing retrieved documents~\citep{yao2025private} and dataset privatization~\citep{zeng2024mitigating}, but these lack formal privacy guarantees and remain vulnerable to strong adversarial attacks. 
A complementary line of work considers protecting user queries in cloud-hosted RAG~\citep{cheng2024remoterag}, which addresses a different threat model than ours. 

For additional related work on the use of differential privacy in large language models and the line of individual privacy accounting, we refer readers to Appendix~\ref{apx:ext_relatedwork}.

\section{Discussion}
\label{sec:discuss}
\textbf{Why Privacy Filter rather than Amplification by Subsampling?}\quad As surveyed in Section~\ref{sec:other_dp_in_llm}, privacy amplification by subsampling \citep{balle2018privacy, wang2019subsampled, pmlrv97zhu19c} is widely used in DP LLM applications, such as DP prompt tuning and DP in-context learning, to enhance generation quality. However, this technique is not well-suited for DP RAG as shown in the experiment section. We would like to discuss the reason behind:
\begin{itemize}[leftmargin=1em]
\item In prompt tuning, the goal is to learn a single task-specific prompt that can generalize to all future queries. In DP in-context learning, a small number of example inputs are selected under DP constraints and reused across queries. In contrast, RAG does not allow for such "unified" prompts or examples: each test-time query requires retrieving and using query-specific documents, which must be handled privately, which makes individual privacy filter a more suitable choice.
\item Moreover, in prompt tuning and in-context learning, all data points in the private dataset can meaningfully contribute to the learned prompt or selected example set. This property enables the use of subsampling-based amplification techniques in algorithm design. In RAG, however, only a sparse subset of documents in the large external corpus are relevant to any given query—most documents provide no utility.
\end{itemize}
These two key differences, the lack of reusable prompts and the sparsity of useful data, motivate the development of our new DP RAG algorithms using R\`enyi filter rather than amplification by sampling.

\textbf{Leveraging Historical QA. }\quad As shown in Table~\ref{tab:precision_of_retrival} and Figure~\ref{fig:hist}, when the relevant documents for different questions exhibit significant overlap, the quality of answers to later questions degrades. This occurs because the documents required to answer the queries may exhaust their privacy budgets and are subsequently filtered out from the active set passed to the RAG algorithm. In the extreme case where a user repeatedly submits the same query, only the first response may retain high quality, while subsequent answers degrade due to the unavailability of relevant documents.

A potential remedy is to reuse historical answers as auxiliary documents in future queries. This can be done without incurring any additional privacy cost, owing to the post-processing property of differential privacy.

\section{Conclusion}\label{sec:conclusion}
We proposed the first differentially private (DP) framework for retrieval-augmented generation (RAG) that supports answering multiple queries while protecting a sensitive external dataset. 
We introduced two algorithms: \dpfixtau\ and \dpadaptovetau\ differ in how they select documents for each query under DP guarantees, which have their advantage for different types of question set.
Through comprehensive experiments on various question datasets and three LLMs, we demonstrated that our methods achieve the utility that outperforms a Non-RAG baseline for answering $100$ questions under a realistic budget of $\varepsilon=10$. We also showed that \dpadaptovetau\ performs particularly well on correlated question sets. 
We hope our contributions provide a foundation for more practical and principled privacy-preserving RAG systems.

\bibliography{iclr2026_conference}
\bibliographystyle{iclr2026_conference}

\newpage

\appendix

\section{Extended Related Work}\label{apx:ext_relatedwork}

\subsection*{Differential Privacy in Large Language Models}\label{sec:other_dp_in_llm}

Beyond our focus on DP for RAG, differential privacy has also been explored in a variety of LLM settings, including pre-training and fine-tuning~\citep{charles2024fine,yu2021differentially,li2021large}, prompt tuning~\citep{duan2023flocks,hong2024dpopt}, and in-context learning~\citep{tang2024privacypreserving,wu2024privacypreserving}. 
These tasks differ structurally and thus require different DP mechanisms. 
In pre-training and fine-tuning, the challenge lies in optimizing model parameters while maintaining stability under DP noise, whereas in RAG, the emphasis is on protecting privacy during inference-time retrieval and generation.  Closer to our setting are DP methods for prompt tuning and in-context learning. 
Still, the structural differences between these tasks and RAG lead to distinct algorithmic requirements (see Section~\ref{sec:discuss} for discussion).  Another line of research investigates differentially private synthetic test generation under varying levels of model access. 
\citet{vinod2025invisibleink, amin2025clustering, amin2024private} focus on next-token prediction with logits access, while \citet{xie2024differentially} studies the API-access setting, which we also include in our comparisons.

\subsection*{Individual Privacy Accounting and Privacy Filters}
Individual privacy accounting tracks the privacy loss of a single data point, often yielding tighter bounds than worst-case analyses over all neighboring datasets \citep{dmns06}. 
This perspective was introduced by \citet{feldman2021individual} in the context of Rényi Differential Privacy and later extended to Gaussian Differential Privacy by \citet{koskela2022individual}. 
See \citet[Section~1.2]{feldman2021individual} for a detailed overview. Within this framework, privacy filters provide a general mechanism for adaptively enforcing privacy constraints by halting an algorithm once the cumulative privacy loss reaches a budget. Individual privacy filters \citep{feldman2021individual, koskela2022individual} refine this idea by operating at the granularity of single data points, excluding them from further computation once their budgets are exhausted. For additional developments and extensions, see \citet{rogers2016privacy, feldman2021individual, koskela2022individual, smith2022fully, whitehouse2023fully}.

\section{Discussion of Sparsity in RAG}\label{apx:sparse}
In retrieval-augmented generation (RAG), relevance is \emph{inherently sparse}: for any given query, only a small subset of the external corpus contains the necessary information, while the vast majority are irrelevant. We illustrate this sparsity with representative examples from the four datasets used in this paper, as shown in Table~\ref{tab:question_example}. For instance, in Natural Questions, the query ``what is the story behind \emph{Five Nights at Freddy's}?'' is mainly supported by the corresponding Wikipedia article.

\begin{table}[ht]
  \centering
\caption{Example questions drawn from official sources: Natural Questions (\href{https://ai.google.com/research/NaturalQuestions/visualization}{visualization page}); TriviaQA (\href{https://nlp.cs.washington.edu/triviaqa/sample.html}{example page}); MQuAKE (\href{https://github.com/princeton-nlp/MQuAKE}{GitHub repository}); and ChatDoctor (\href{https://huggingface.co/zl111/ChatDoctor}{Hugging Face page}).}
  \label{tab:question_example}
\resizebox{\textwidth}{!}{
  \begin{tabular}{l | l} 
    \toprule
    Dataset  & Example Question \\
    \midrule
    Natural Questions & what is the story behind 5 nights at freddy's  \\
    \midrule
    TriviaQA     &  Miami Beach in Florida borders which ocean?	   \\
    \midrule
    MQuAKE        & \makecell[l]{What country is the birthplace of the sport associated with Hampshire Cricket Board? \\ Where was the sport associated with Hampshire Cricket Board originated? Which \\ country is credited with creating the sport associated with Hampshire Cricket Board?}  \\
    \midrule
    ChatDoctor   & \makecell[l]{"instruction": "If you are a doctor, please answer the medical questions based on \\ the patient's description." "input": "Doctor, I think I've been poisoned. I drank \\ some ethylene glycol by mistake. "}
   \\
    \bottomrule
  \end{tabular}
}
\end{table}

\newpage
\section{Supplementary Algorithms}\label{apx:old_rag}
This section contains additional algorithms that were excluded from the main body of the paper for space reasons.
\subsection{Auxiliary Algorithms}
\textbf{(Top-K selection)} Algorithm~\ref{alg:topk} selects the top-$K$ documents from the dataset $D$ according to the score function $r$. If $|D| < K$, it pads the output with empty strings so that the result always contains exactly $K$ elements, as required for the privacy accounting (see Lemma~\ref{lem:privacy_batch_rag_ada}).
\begin{algorithm}[H]
\caption{$\topk(D,K,r)$}\label{alg:topk}
\setcounter{AlgoLine}{0}
\KwIn{dataset $D$, sample size $K$, score function $r$}

\eIf{$|D| \geq K$}{
    $D^\prime \leftarrow$ top-$K$ documents from $D$ ranked by $r$ \tcp*[r]{assume no ties}
}{
    $D^\prime \leftarrow D \cup \{\texttt{""}\}^{\,K - |D|}$ \tcp*[r]{pad with empty strings to size $K$}
}
\KwRet{$D^\prime$}
\end{algorithm}

\textbf{(Poisson Subsampling)}  Algorithm~\ref{alg:poisson-subsample-dp} implements Poisson subsampling: it independently includes each data point $z_i \in D$ in the subsample $S$ with probability $\gamma$, resulting in a (random) subset whose expected size is $\gamma n$.
\begin{algorithm}[H]
\caption{\poissionsampling $(D, \gamma)$}
\label{alg:poisson-subsample-dp}
\SetKwInOut{Input}{Input}\SetKwInOut{Output}{Output}
\DontPrintSemicolon
\Input{Dataset $D=\{z_1,\dots,z_n\}$, sampling rate $\gamma \in(0,1)$}
$S \leftarrow \emptyset$\tcp*[r]{Initialize subsample}
\For{$i \gets 1$ \KwTo $n$}{
    Draw $b_i \sim \mathrm{Bernoulli}(\gamma)$\;
    \If{$b_i = 1$}{
        $S \leftarrow S \cup \{z_i\}$\;
    }
}
\KwRet $S$
\end{algorithm}

\textbf{(Token Counting)} Algorithm~\ref{alg:counting} computes the token count vector over a fixed vocabulary: given a (multi)set of tokens $S$ and vocabulary $\mathcal{V}$, it iterates over each vocabulary item $v_j$ and counts how many times $v_j$ appears in $S$, returning the resulting count vector $\vec{u} \in \mathbb{N}^{|\mathcal{V}|}$.
\begin{algorithm}[H]
\caption{\algocount$(S, \mathcal{V})$}
\label{alg:counting}
\setcounter{AlgoLine}{0}
\KwIn{
  A (multi)set of tokens $S \in \mathcal{V}^*$, a vocabulary $\mathcal{V} = \{v_1, v_2, \ldots, v_{|\mathcal{V}|}\}$.
}
\KwOut{
  Count vector $\vec{u} \in \mathbb{N}^{|\mathcal{V}|}$, where $u_j$ is the number of times $v_j$ appears in $S$.
}

\For{$j \in \{1,2,\ldots,|\mathcal{V}|\}$}{
  $u_j \leftarrow \sum_{x \in S} \mathbf{1}\{x = v_j\}$\;
}

\KwRet $\vec{u}$\;

\end{algorithm}

\textbf{(Exponential Mechanism)} Algorithm~\ref{alg:expmech} implements the exponential mechanism: given a candidate set $\mathcal{V}$ and utility scores ${u_j}$ with sensitivity $\Delta u$, it assigns each candidate $v_j$ an unnormalized weight $\exp\bigl(\frac{\varepsilon u_j}{2\Delta u}\bigr)$, normalizes these to probabilities, and then samples an output $v_J$ from the resulting categorical distribution, ensuring $\varepsilon$-DP.
\begin{algorithm}[ht]
\caption{\expomech$(\vec{u}, \mathcal{V}, \varepsilon)$}
\label{alg:expmech}
\setcounter{AlgoLine}{0}
\KwIn{
  Candidate set $\mathcal{V} = \{v_1, v_2, \ldots, v_{|\mathcal{V}|}\}$, privacy parameter $\varepsilon$, utility scores $\vec{u} = (u_1,\ldots,u_{|\mathcal{V}|})$ with $u_j := u(v_j)$
}
\KwOut{A selected element $v \in \mathcal{V}$}

\For{$j \in \{1,2,\ldots,|\mathcal{V}|\}$}{
  $w_j \leftarrow \exp\!\left(\frac{\varepsilon \cdot u_j}{2 \Delta u}\right)$ \tcp*[r]{unnormalized weight}
}

$Z \leftarrow \sum_{j=1}^{|\mathcal{V}|} w_j$ \tcp*[r]{normalizer / partition function}

\For{$j \in \{1,2,\ldots,|\mathcal{V}|\}$}{
  $p_j \leftarrow w_j / Z$ \tcp*[r]{sampling probability for $v_j$}
}

Sample $J \sim \text{Categorical}(p_1,\ldots,p_{|\mathcal{V}|})$\;

\KwRet $v_J$\;

\end{algorithm}

\subsection{Differentially Private RAG for Single-Query Question Answering}

\textbf{(\dprag)} Algorithm~\ref{alg:dp-rag-v2} describes our differentially private RAG procedure for single-question answering: at each decoding step, it compares a baseline token (without retrieval) to votes from $m$ RAG “voters” over disjoint document subsets, uses a noisy threshold test (via Laplace noise) to decide whether retrieval can be used, and when it does, privately selects the next token with the exponential mechanism under a per-token budget $\varepsilon_0$, stopping when either an $\langle\mathtt{EOS}\rangle$ token is generated or the total privacy budget $\varepsilon$ is exhausted. Algorithm~\ref{alg:dp-rag-v2} can be seen as a variant of \citet[Algorithm~2]{koga2024privacy}, where the LimitedDomain mechanism~\citep{durfee2019practical} is replaced by the exponential mechanism in the private token-generation step, yielding a stronger pure-DP guarantee and simplifying the privacy analysis.

\begin{algorithm}[H]
\caption{\dprag$(x, D, \mathrm{LLM}, \varepsilon)$ }
\label{alg:dp-rag-v2}
\setcounter{AlgoLine}{0}
\KwIn{Prompt $x$; document collection $D$; language model $\mathrm{LLM}$; total budget $\varepsilon$.}
\KwRequire{Per-token budget $\varepsilon_0$; max tokens $T_{\max}$; voters $m$; docs per voter $k$; retriever $R$; vote threshold $\theta$.}
\KwSet{$\varepsilon_{\text{Lap}} \leftarrow \varepsilon_{\text{Expo}} \leftarrow \varepsilon_0/2$; discoveries left $c \leftarrow \lfloor \varepsilon/\varepsilon_0 \rfloor$}

\smallskip
$\hat{\theta} \leftarrow \theta + \mathrm{Lap}(2/\varepsilon_{\text{Lap}})$ \tcp*[r]{noisy threshold}
$D_x \leftarrow R(x, D; m k)$; split $D_x$ uniformly into $m$ chunks $\{D_x^{(i)}\}_{i=1}^m$.

\For{$t \leftarrow 1$ \KwTo $T_{\max}$}{
  $b \leftarrow \mathrm{LLM}(x, \emptyset \mid y_{<t})$ \tcp*[r]{baseline token (no RAG)}
  \For{$i \leftarrow 1$ \KwTo $m$}{
    $v_i \leftarrow \mathrm{LLM}(x, D_x^{(i)} \mid y_{<t})$
  }
  $\vec{u} \leftarrow \algocount(\{v_i\}_{i=1}^{m}, \mathcal{V})$; \quad $s \leftarrow H[b]$ \tcp*[r]{Algorithm~\ref{alg:counting}}
  \uIf{$s + \mathrm{Lap}(4/\varepsilon_{\mathrm{Lap}}) \le \hat{\theta}$}{
    $y_t \leftarrow \expomech(\vec{u}, \mathcal{V}, \varepsilon_{\text{Expo}})$ \tcp*[r]{Algorithm~\ref{alg:expmech}}
    $c \leftarrow c - 1$ \\
    $\hat{\theta} \leftarrow \theta + \mathrm{Lap}(2/\varepsilon_{\text{Lap}})$
  }\Else{
    $y_t \leftarrow b$ \tcp*[r]{keep baseline}
  }
  \If{$y_t=\langle\mathtt{EOS}\rangle$ \textbf{or} $c=0$}{\KwRet $(y_1,\ldots,y_t)$}
}
\KwRet $(y_1,\ldots,y_{T_{\max}})$
\end{algorithm}

We now give the privacy guarantee for Algorithm~\ref{alg:dp-rag-v2}.

\begin{lemma}[Privacy Guarantee for Algorithm~\ref{alg:dp-rag-v2}]\label{lem:priv_base_dp_rag}
    Algorithm~\ref{alg:dp-rag-v2} satisfies $\varepsilon$-DP under add/remove relationship.
\end{lemma}

\begin{proof}
Notice that Algorithm~\ref{alg:dp-rag-v2} is an instantiation of AboveThreshold (\citet[Algorithm~1]{dwork2014algorithmic}) with at most $c$ discoveries. It therefore suffices to show that each discovery event (i.e., each use of the exponential mechanism) satisfies $\varepsilon_0$-DP, where $c = \lfloor \varepsilon / \varepsilon_0 \rfloor$.

We first verify that the added noise meets the requirements of the stated privacy guarantee, namely for the threshold perturbation (Line~8 of Algorithm~\ref{alg:dp-rag-v2}) and for the exponential mechanism (Line~9 of Algorithm~\ref{alg:dp-rag-v2}). Without loss of generality, assume the input document set has size larger than $mk$. Consider two neighboring datasets $D$ and $D'$ such that $\lvert D \setminus D' \rvert + \lvert D' \setminus D \rvert \le 1$. This implies $\lvert D_x \setminus D_x' \rvert + \lvert D_x' \setminus D_x \rvert \le 2$, since the retriever $R$ ranks documents by relevance and selects the top-$mk$ entries. Replacing a single token in the voting results can change at most one bin count in the histogram by $1$, so the score function satisfies
\[
  \lvert s(D) - s(D') \rvert \;=\; \lvert s(D_x) - s(D_x') \rvert \;\le\; 1,
\]
where $s$ is defined in Line~7 of Algorithm~\ref{alg:dp-rag-v2}. Similarly, the utility function has unit sensitivity for each token, i.e.,
\[
  \lvert u_j(D_x) - u_j(D_x') \rvert \le 1,\quad \forall j \in \{1,\ldots,|\mathcal{V}|\},
\]
where $u_j$ is the $j$-th coordinate of $\vec{u}$.

Thus, by \citet[Theorem~3.23]{dwork2014algorithmic} and adaptive composition, each discovery is $\varepsilon_0$-DP. Since the number of discoveries satisfies $c = \lfloor \varepsilon / \varepsilon_0 \rfloor$, basic composition implies that the entire execution of Algorithm~\ref{alg:dp-rag-v2} satisfies $\varepsilon$-DP.
\end{proof}

\subsection{Baseline Algorithms for Differentially Private Multi-Query RAG}

\textbf{(\naivedprag)} Algorithm~\ref{alg:naive_batch_dp_rag} defines a naïve baseline for DP multi-query RAG: it answers each query $q_t$ independently by invoking the single-query DP-RAG procedure (Algorithm~\ref{alg:dp-rag-v2}) on the private dataset $D$ with per-query budget $\varepsilon_q$, yielding responses $\{a_t\}_{t=1}^T$.

\begin{algorithm}[H]
\caption{\naivedprag}
\label{alg:naive_batch_dp_rag}
\setcounter{AlgoLine}{0}
\KwIn{Private external dataset $D$, query sequence $\{q_1, q_2, \ldots, q_T\}$, per-query budget $\varepsilon_q$}

\For{$t = 1, \ldots, T$}{
    $a_t \leftarrow \dprag(q_t, D, \llm, \varepsilon_q)$ \tcp*[r]{Apply Algorithm~\ref{alg:dp-rag-v2}}
}

\KwRet{$(a_1, a_2, \ldots, a_T)$}
\end{algorithm}

\begin{lemma}[Privacy guarantee of Algorithm~\ref{alg:naive_batch_dp_rag}]
    Algorithm~\ref{alg:naive_batch_dp_rag} satisfies $T\varepsilon_q$-DP under add/remove relationship.
\end{lemma}

\begin{proof}
By Lemma~\ref{lem:priv_base_dp_rag}, every call to \dprag{} satisfies $\varepsilon_q$-DP. Applying basic composition~\citep{dwork2014algorithmic} to $T$ such calls introduces an extra factor of $T$ in the privacy bound.
\end{proof}

\textbf{(\sampledprag)} Algorithm~\ref{alg:naive_batch_dp_rag_subsample} defines a baseline for DP multi-query RAG using subsampling: for each query $q_t$, it first applies Poisson subsampling to the private dataset $D$ with rate $\gamma$ to obtain $D_t$, then runs the single-query DP-RAG procedure (Algorithm~\ref{alg:dp-rag-v2}) on $(q_t, D_t)$ with per-query budget $\varepsilon_q$, producing answers $(a_1,\ldots,a_T)$.

\begin{algorithm}[H]
\caption{\sampledprag}\label{alg:naive_batch_dp_rag_subsample}
\setcounter{AlgoLine}{0}
\KwIn{Private external dataset $D$, query sequence $\{q_1, q_2, \ldots, q_T\}$, 
      per-query privacy budget $\varepsilon_q$, Poisson sampling rate $\gamma$}
\For{$t = 1, \ldots, T$}{
    $D_t \leftarrow \poissionsampling(D, \gamma)$ \tcp*[r]{Apply Algorithm~\ref{alg:poisson-subsample-dp}}
    $a_t \leftarrow \dprag(q_t, D_t , \llm, \varepsilon_q)$ \tcp*[r]{Apply Algorithm~\ref{alg:dp-rag-v2}}
}

\KwRet{$(a_1, a_2, \ldots, a_T)$}
\end{algorithm}

\begin{lemma}
    Algorithm~\ref{alg:naive_batch_dp_rag_subsample} satisfies $T\times\log(1 + \gamma (e^{\varepsilon_q }-1))$-DP under add/remove neigh
\end{lemma}

\begin{proof}
    Since each call of the DP-RAG satisfies $\varepsilon_q$-DP, by \citet[Theorem~8]{balle2018privacy}, the Poisson subsampled DP-RAG satisfies $\log(1 + \gamma (e^{\varepsilon_q}-1))$-DP. Applying basic composition~\citep{dwork2014algorithmic} to $T$ such calls introduces an extra factor of $T$ in the privacy bound.
\end{proof}

\section{Privacy Guarantee of Differentially Private Multi-Query RAG Algorithms}
\subsection*{Privacy Guarantee for Algorithm~\ref{alg:batch_dp_rag-v2}}\label{apx:privacy_proof}
\begin{thm*}[Restatement of Theorem~\ref{lem:privacy_batch_rag_ada}]
    $\dpadaptovetau$ (Algorithm~\ref{alg:batch_dp_rag-v2}) satisfies $\varepsilon$-differential privacy under the add/remove neighboring relation, provided that the ex-ante individual privacy budget of every $z \in D$ is at most $\varepsilon$.
\end{thm*}

\begin{proof}
The proof follows the approach of \citet[Theorem~4.5]{feldman2021individual}. We first bound the individual privacy loss of the $t$-th prefix-sum release algorithm, denoted by $\mathcal{A}_t$. 
Consider $S, \tilde{S} \in \mathcal{S}(z_i, n)$, and without loss of generality assume $z_i \in S$. 
\emph{Conditioned on the trajectory $r^{(t-1)}$ from the previous $t-1$ rounds}, 
for any possible output sequence $b^{(q)} := (b_1, b_2, \ldots, b_q)$ with $q \leq B$, 
the only interesting regime is when there exists $j \in [q]$ such that $z_i$ contributes to $b_j$. 
Otherwise, we have
\[
\mathcal{A}_t(S \mid r^{(t-1)}) \stackrel{d}{=} \mathcal{A}_t(\tilde{S} \mid r^{(t-1)}).
\]
In the former case, we can perform the decomposition using Bayes’ rule:
\begin{equation*}
    \begin{aligned}
        \log \left( \frac{\mathbb{P}(\mathcal{A}_t (S)=b^{(q)})}{\mathbb{P}(\mathcal{A}_{t}(\Tilde{S})=b^{(q)})}  \right) &= \underbrace{\log\left(\frac{\mathbb{P}(\mathcal{A}_t(S)[j+1:q]=b^{(j+1:q)} \mid b^{(j)})}{\mathbb{P}(\mathcal{A}_t(\Tilde{S})[j+1:q]=b^{(j+1:q)} \mid b^{(j)})} \right)}_{(a)} \\
        &\quad\quad + \underbrace{\log\left(\frac{\mathbb{P}(\mathcal{A}_t(S)[j]=b_j \mid b^{(j-1)})}{\mathbb{P}(\mathcal{A}_t(\Tilde{S})[j]=b_j \mid b^{(j-1)})} \right)}_{(b)}+ \underbrace{\log\left( \frac{\mathbb{P}(\mathcal{A}_t(S)=b^{(j-1)})}{\mathbb{P}(\mathcal{A}_t(\tilde{S})=b^{(j-1)})} \right)}_{(c)}  \\
        &\leq \varepsilon_{\mathrm{thr}}
    \end{aligned}
\end{equation*}

Observe that the bins are disjoint, which implies that the privacy budget consumption is independent across different data points. Consequently, we have $(a)=(c)=0$ and $(b)\leq \varepsilon_{\mathrm{thr}}$.

Next, consider the RAG step. The non-trivial case arises when $z_i \in A_t'$. In this case, by the composition theorem, the privacy loss of $\dprag \circ \topk$ is bounded above by $\varepsilon_{\mathrm{RAG}}$. 

Moreover, $\mathcal{E}(z_i)$ constitutes a valid stopping time, as the privacy budget is updated after each invocation of the algorithms, and $z_i$ is only used when its budget remains sufficient. Therefore, by \citet[Corollary~3.3]{feldman2021individual}, the overall privacy guarantee is given by $\mathcal{E}(z)$, which is upper bounded by $\varepsilon$.
\end{proof}

\begin{remark}

Algorithm~\ref{alg:batch_dp_rag-v2} employs a fixed, data-independent threshold $k$ (Line~9), rather than a data-dependent choice such as a DP quantile. If, instead, we were to use a privately released data-dependent threshold, the resulting selection would become coupled to the data, thereby violating the assumptions underlying the individual-filter guarantee.
\end{remark}

\subsection*{Privacy Guarantee for Algorithm~\ref{alg:dp_fix_tau}}
\begin{thm*}[Restatement of Theorem~\ref{lem:privacy_batch_rag_fix}]
    $\dpfixtau$ satisfies \(\varepsilon\)-differential privacy if, for every $ z \in D $, the ex-ante individual privacy budget is at most \(\varepsilon\).
\end{thm*}
\begin{proof}
Since $\mathcal{E}(z) \leq \varepsilon$ for every $z \in D$, by an analysis analogous to the proof of Theorem~\ref{lem:privacy_batch_rag_ada}, the claimed privacy guarantee follows directly from \citet[Corollary~3.3]{feldman2021individual}.
\end{proof}

\section{Experimental Details}
\label{app:exp}
\textbf{Implementation details of our methods and baseline methods.} All four DP algorithms rely on shared hyperparameters from DPSparseVoteRAG, including the number of retrieved documents $k$, the per-token privacy budget $\varepsilon_{\text{token}}$, and the SVT threshold $\tau_{\text{svt}}$. Following \citet{koga2024privacy}, we evaluate each method under a grid of settings with $k \in \{30, 40, 50\}$, $\varepsilon_{\text{token}} \in \{0.5, 1.0, 2.0\}$, and $\tau_{\text{svt}} = k/2$. 
For MURAG-ADA, the bins for discretizatizing the similarity scores are the bins between $70$ and $100$ with the bin size $0.2$.
For the Non-Private-RAG, we retrieve \{1, 3, 5, 10\} documents in the context for each question.
We report the best performance for each method over these configurations.
For PE, we adopt the same hyperparameter configuration used for unconditional generation on the PubMed dataset (Table 14 in \citet{xie2024differentially}) and generate $2000$ synthetic samples.

\paragraph{Implementation details of the membership inference attack.} We implement the Interrogation Attack (IA; \citet{naseh2025riddle}) to assess the privacy protection offered by our methods. For each test document $x$, the attack generates $m=30$ tailored questions. Following the setup in \citet{naseh2025riddle}, we reuse their prompt template (Figure 14 in their paper) and employ GPT-4 to generate the questions for documents in the ChatDoctor dataset.

The following is an example of one target document and three generated questions.
\begin{tcolorbox}[
    colback=gray!5,
    colframe=blue!75!black,
    fonttitle=\bfseries
]
Document: Patient's description: I have noticed that my husband (black, male 47 yrs old) has increasingly began involuntary head movements (as if his right ear is trying to touch his right shoulder). They are sporadic in nature and he is aware that they happen but is not able to control them.  Can you please help explain what may be causing this?  I have notice them happening more often than not. \#\#\# Doctor's answer: Hi, some people get this increase of such movements in time of stress, so make him relaxed, you can observe such movements when he is tensed due to some reasons, meditation could control that. Wish him good health.\\
\\
Generated Questions and Corresponding Answers:\\
Q: Did the patient experience sharp abdominal pain on the right side? A: Yes \\
Q: Is there increased pain on the lower right side during and after sex? A: Yes \\
Q: Did the patient report pain only during sex, not afterwards? A: No \\
\end{tcolorbox}
The construction of query, a combination of generated question and the necessary context, is as follow
\begin{tcolorbox}[
    colback=gray!5,
    colframe=blue!75!black,
    fonttitle=\bfseries
]
I am inquiring about a patient-doctor conversation where the doctor talked "\{doctor's answer in the target document\}" \{generated question\} Please answer with 'Yes,' or 'No'.
\end{tcolorbox}

\paragraph{QA templates.} For Trivia, NQ, and MQuAKE QA datasets, we adopt the following template for Non-RAG QA
\begin{tcolorbox}[
    colback=gray!5,
    colframe=blue!75!black,
    fonttitle=\bfseries
]
Instruction: Give a simple short answer for the question\\
Question: \{question\}\\
Answer:
\end{tcolorbox}
and the following template for RAG QA
\begin{tcolorbox}[
    colback=gray!5,
    colframe=blue!75!black,
    fonttitle=\bfseries
]
Instruction: Give a simple short answer for the question based on the context\\
Context: \{document $1$; $\cdots$; document $m$\}.
Question: \{question\}\\
Answer:
\end{tcolorbox}
For ChatDoctor dataset,  we adopt the following template for Non-RAG QA
\begin{tcolorbox}[
    colback=gray!5,
    colframe=blue!75!black,
    fonttitle=\bfseries
]
Instruction: if you are a doctor, please answer the medical questions based on the patient's description\\
Question: \{question\}\\
Answer:
\end{tcolorbox}
and the following template for RAG QA
\begin{tcolorbox}[
    colback=gray!5,
    colframe=blue!75!black,
    fonttitle=\bfseries
]
Instruction: if you are a doctor, please answer the medical questions based on the patient's description and the given example\\
Example: \{document $1$; $\cdots$; document $m$\}.
Question: \{question\}\\
Answer:
\end{tcolorbox}

\paragraph{Implementation details of the private evolution (PE, ~\citep{xie2024differentially}).} 
Since the external datasets used in our RAG setup are quite large, applying a synthetic text generation method directly on these private datasets can be computationally inefficient. To alleviate this overhead—and to give the baseline a favorable setup—we adopt an approximation: for each QA dataset, we select the top-50 document for each question and attain a joint document set. Then we run PE on this smaller but question-focused subset of the private dataset.

In our experiment, we are using the following prompts for the random API and variation API as follows:

\paragraph{Random API} For the ChatDoctor dataset, we adopt the following template:
\begin{tcolorbox}[
    colback=gray!5,
    colframe=blue!75!black,
    fonttitle=\bfseries
]
Instruction: \{example\} Using a variety of sentence structures, write a dialogue between a patient describing their condition and a doctor giving suggestions\\
Answer:
\end{tcolorbox}
and the following template for Trivia, NQ, and MQuAKE QA datasets
\begin{tcolorbox}[
    colback=gray!5,
    colframe=blue!75!black,
    fonttitle=\bfseries
]
Instruction: Using a variety of sentence structures, for answering the question \{question\}, write a Wikipedia paragraph\\
Answer:
\end{tcolorbox}

In the ChatDoctor random API template, the placeholder ${\text{example}}$ is filled with a sample dialogue in which a patient describes their condition and a doctor provides suggestions. In contrast, the random API templates for Trivia, NQ, and MQuAKE use the placeholder ${\text{question}}$, sampled from the corresponding question set in a round-robin manner. As the number of API calls exceeds the set size, the sampling ensures every question is used at least once, guaranteeing full coverage in the PE generation.

\paragraph{Variation API} For the ChatDoctor dataset, we adopt the following template:
\begin{tcolorbox}[
    colback=gray!5,
    colframe=blue!75!black,
    fonttitle=\bfseries
]
Instruction: Please rephrase the following {tone}sentences as a dialogue between a patient describing their condition and a doctor giving suggestions\\
Answer:
\end{tcolorbox}
and the following template for Trivia, NQ, and MQuAKE QA datasets
\begin{tcolorbox}[
    colback=gray!5,
    colframe=blue!75!black,
    fonttitle=\bfseries
]
Instruction: Please rephrase the following sentences as a Wikipedia paragraph\\
Answer:
\end{tcolorbox}

\end{document}